\documentclass{article} % For LaTeX2e
\usepackage[table]{xcolor}
\usepackage{iclr2022_conference,times}

\usepackage[utf8]{inputenc} % allow utf-8 input
\usepackage[T1]{fontenc}    % use 8-bit T1 fonts
\usepackage{hyperref}       % hyperlinks
\usepackage{url}            % simple URL typesetting
\usepackage{booktabs}       % professional-quality tables
\usepackage{amsfonts}       % blackboard math symbols
\usepackage{nicefrac}       % compact symbols for 1/2, etc.
\usepackage{microtype}      % microtypography
\usepackage{amsmath,amsfonts,amssymb,amsthm}
\usepackage{standalone}
\usepackage{algorithm}
\usepackage{algpseudocode}
\usepackage{bbm}
\usepackage{multirow}
\usepackage{adjustbox}
\usepackage[makeroom]{cancel}

\usepackage{pifont}
\newcommand{\cmark}{\ding{51}}%
\newcommand{\xmark}{\ding{55}}%

\definecolor{flat-red}{RGB}{231, 76, 60}
\definecolor{flat-blue}{RGB}{41, 128, 185}
\definecolor{flat-teal}{RGB}{22, 160, 133}
\definecolor{flat-green}{RGB}{39, 174, 96}
\definecolor{flat-purple}{RGB}{142, 68, 173}
\definecolor{flat-yellow}{RGB}{243, 156, 18}

\definecolor{C2}{RGB}{44, 160, 44}
\definecolor{C3}{RGB}{214, 39, 40}

\usepackage[frozencache,newfloat]{minted}
\usepackage{caption}

\newenvironment{code}{\captionsetup{type=listing}}{}
\SetupFloatingEnvironment{listing}{name=Code Snippet}

\usepackage{amsmath,amsfonts,bm}

\def\eqref#1{equation~\ref{#1}}
\def\1{\bm{1}}

\def\vzero{{\bm{0}}}

\def\vtheta{{\bm{\theta}}}
\def\va{{\bm{a}}}

\def\vh{{\bm{h}}}

\def\vp{{\bm{p}}}

\def\vs{{\bm{s}}}

\def\vx{{\bm{x}}}
\def\vy{{\bm{y}}}
\def\vz{{\bm{z}}}

\def\mA{{\bm{A}}}
\def\mB{{\bm{B}}}
\def\mC{{\bm{C}}}
\def\mD{{\bm{D}}}

\def\mF{{\bm{F}}}
\def\mG{{\bm{G}}}

\def\mI{{\bm{I}}}

\def\mV{{\bm{V}}}
\def\mW{{\bm{W}}}

\DeclareMathAlphabet{\mathsfit}{\encodingdefault}{\sfdefault}{m}{sl}
\SetMathAlphabet{\mathsfit}{bold}{\encodingdefault}{\sfdefault}{bx}{n}

\def\gB{{\mathcal{B}}}

\def\gL{{\mathcal{L}}}

\def\gO{{\mathcal{O}}}

\def\gS{{\mathcal{S}}}

\def\sR{{\mathbb{R}}}

\newcommand{\E}{\mathbb{E}}

\newcommand{\softmax}{\mathrm{softmax}}

\newtheorem{lemma}{Lemma}
\newtheorem{proposition}{Proposition}
\newtheorem{cor}{Corollary}

\newtheorem{definition}{Definition}

\definecolor{darkblue}{rgb}{0.0, 0.0, 0.55}
\hypersetup{
	pdftitle={Continuous-Time Meta-Learning with Forward Mode Differentiation},
	pdfkeywords={comln,meta-learning,few-shot learning,dynamical systems,ode},
	pdfborder=0 0 0,
	pdfpagemode=UseNone,
	colorlinks=true,
	linkcolor=darkblue,
	citecolor=darkblue,
	filecolor=darkblue,
	urlcolor=darkblue,
	pdfview=FitH,
	pdfauthor={Tristan Deleu, David Kanaa, Leo Feng, Giancarlo Kerg, Yoshua Bengio, Guillaume Lajoie, Pierre-Luc Bacon}
}

\usepackage{hyperref}
\usepackage[capitalize]{cleveref}

\crefformat{section}{#2Section~#1#3}
\crefformat{appendix}{#2Appendix~#1#3}

\crefformat{equation}{(#2#1#3)}
\crefmultiformat{equation}{#2Equations~#1#3}%
{ \&~#2#1#3}{, #2#1#3}{, \&~#2#1#3}

\crefformat{table}{#2Table~#1#3}
\crefformat{figure}{#2Figure~#1#3}

\crefformat{theorem}{#2Theorem~#1#3}
\crefmultiformat{theorem}{#2Theorems~#1#3}%
{ \&~#2#1#3}{, #2#1#3}{, and~(#2#1#3)}

\crefformat{lemma}{#2Lemma~#1#3}
\crefformat{proposition}{#2Proposition~#1#3}
\crefmultiformat{proposition}{#2Propositions~#1#3}%
{ \&~#2#1#3}{, #2#1#3}{, and~(#2#1#3)}

\crefformat{algorithm}{#2Algorithm~#1#3}
\crefmultiformat{algorithm}{#2Algorithms~#1#3}%
{ \&~#2#1#3}{, #2#1#3}{, and~(#2#1#3)}

\crefformat{corollary}{#2Corollary~#1#3}
\crefformat{definition}{#2Definition~#1#3}

\crefformat{assumption}{#2Assumption~#1#3}
\crefmultiformat{assumption}{#2Assumptions~#1#3}%
{ \&~#2#1#3}{, #2#1#3}{, and~(#2#1#3)}

\DeclareMathOperator{\diag}{diag}
\def\vphi{{\bm{\phi}}}

\newcommand{\Dtrain}[1][\tau]{\mathcal{D}_{#1}^{\mathrm{train}}}
\newcommand{\Dtest}[1][\tau]{\mathcal{D}_{#1}^{\mathrm{test}}}

\title{Continuous-Time Meta-Learning with\\Forward Mode Differentiation}

\author{Tristan Deleu\thanks{Correspondence to: Tristan Deleu <\href{mailto:deleutri@mila.quebec}{deleutri@mila.quebec}>\newline\hspace*{1.35em}\textsuperscript{1}CIFAR Senior Fellow,\ \textsuperscript{2}CIFAR AI Chair\newline\hspace*{1.75em}Code is available at: \href{https://github.com/tristandeleu/jax-comln}{\texttt{https://github.com/tristandeleu/jax-comln}}} \quad David Kanaa \quad Leo Feng \quad Giancarlo Kerg\\[0.2em]
\textbf{Yoshua Bengio}\,\textsuperscript{1,2} \quad \textbf{Guillaume Lajoie}\,\textsuperscript{2} \quad \textbf{Pierre-Luc Bacon}\,\textsuperscript{2}\\[0.5em]
Mila -- Universit\'{e} de Montr\'{e}al
}

\usepackage{tikz}
\usepackage{pgfplots}
\usepgfplotslibrary{groupplots}
\usetikzlibrary{positioning}
\usetikzlibrary{arrows}
\usetikzlibrary{calc,fit}
\usetikzlibrary{shapes.geometric}
\usetikzlibrary{shapes.misc}
\usetikzlibrary{decorations.pathmorphing}
\usetikzlibrary{decorations.pathreplacing}

\definecolor{tableau10_C0}{RGB}{31, 119, 180}
\definecolor{tableau10_C1}{RGB}{255, 127, 14}
\definecolor{tableau10_C2}{RGB}{44, 160, 44}
\definecolor{tableau10_C3}{RGB}{214, 39, 40}

\iclrfinalcopy % Uncomment for camera-ready version, but NOT for submission.
\begin{document}

\maketitle

\begin{abstract}
Drawing inspiration from gradient-based meta-learning methods with infinitely small gradient steps, we introduce Continuous-Time Meta-Learning (COMLN), a meta-learning algorithm where adaptation follows the dynamics of a gradient vector field. Specifically, representations of the inputs are meta-learned such that a task-specific linear classifier is obtained as a solution of an ordinary differential equation (ODE). Treating the learning process as an ODE offers the notable advantage that the length of the trajectory is now continuous, as opposed to a fixed and discrete number of gradient steps. As a consequence, we can optimize the amount of adaptation necessary to solve a new task using stochastic gradient descent, in addition to learning the initial conditions as is standard practice in gradient-based meta-learning. Importantly, in order to compute the exact meta-gradients required for the outer-loop updates, we  devise an efficient algorithm based on forward mode differentiation, whose memory requirements do not scale with the length of the learning trajectory, thus allowing longer adaptation in constant memory. We provide analytical guarantees for the stability of COMLN, we show empirically its efficiency in terms of runtime and memory usage, and we illustrate its effectiveness on a range of few-shot image classification problems.
\end{abstract}

\section{Introduction}
\label{sec:introduction}
Among the existing meta-learning algorithms, gradient-based methods as popularized by Model-Agnostic Meta-Learning \citep[MAML, ][]{finn17maml} have received a lot of attention over the past few years. They formulate the problem of learning a new task as an inner optimization problem, typically based on a few steps of gradient descent. An outer \emph{meta}-optimization problem is then responsible for updating the meta-parameters of this learning process, such as the initialization of the gradient descent procedure. However since the updates at the outer level typically require backpropagating through the learning process, this class of methods has often been limited to only a few gradient steps of adaptation, due to memory constraints. Although solutions have been proposed to alleviate the memory requirements of these algorithms, including checkpointing \citep{baranchuk2019memoryefficientmaml}, using implicit differentiation \citep{rajeswaran2019imaml}, or reformulating the meta-learning objective \citep{flennerhag2018leap}, they are generally either more computationally demanding, or only approximate the gradients of the meta-learning objective \citep{nichol2018reptile,flennerhag2020warpgrad}.

In this work, we propose a continuous-time formulation of gradient-based meta-learning, called \emph{Continuous-Time Meta-Learning} (COMLN), where the adaptation is the solution of a differential equation (see \cref{fig:comln}). Moving to continuous time allows us to devise a novel algorithm, based on forward mode differentiation, to efficiently compute the exact gradients for meta-optimization, no matter how long the adaptation to a new task might be. We show that using forward mode differentiation leads to a stable algorithm, unlike the counterpart of backpropagation in continuous time called the adjoint method (frequently used in the Neural ODE literature; \citealp{chen2018neuralode}) which tends to be unstable in conjunction with gradient vector fields. Moreover as the length of the adaptation trajectory is a continuous quantity, as opposed to a discrete number of gradient steps fixed ahead of time, we can treat the amount of adaptation in COMLN as a meta-parameter---on par with the initialization---which we can meta-optimize using stochastic gradient descent. We verify empirically that our method is both computationally and memory efficient, and we show that COMLN outperforms other standard meta-learning algorithms on few-shot image classification datasets.

\section{Background}
\label{sec:background}
In this work, we consider the problem of few-shot classification, that is the problem of learning a classification model with only a small number of training examples. More precisely for a classification task $\tau$, we assume that we have access to a (small) training dataset $\Dtrain = \{(\vx_{m}, \vy_{m})\}_{m=1}^{M}$ to fit a model on task $\tau$, and a distinct test dataset $\Dtest$ to evaluate how well this adapted model generalizes on that task. In the few-shot learning literature, it is standard to consider the problem of $k$-shot $N$-way classification, meaning that the model has to classify among $N$ possible classes, and there are only $k$ examples of each class in $\Dtrain$, so that overall the number of training examples is $M = kN$. We use the convention that the target labels $\vy_{m} \in \{0, 1\}^{N}$ are one-hot vectors.

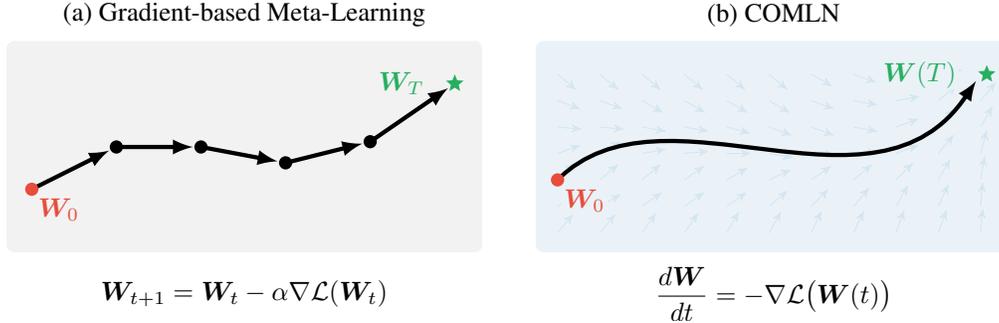
\begin{figure}[t]
    \centering
    \begin{tikzpicture}[every node/.style={inner sep=0pt}, var/.style={draw=black, minimum size=15pt, circle, font=\small, thick, anchor=center}, factor/.style={fill=black, minimum size=5pt, rectangle, anchor=center}, edge/.style={thick, -latex}, modulation/.style={thick, -|}, zlabel/.style={font=\scriptsize, yshift=-2pt}, observed/.style={fill=gray!30}, flowlabel/.style={midway, font=\scriptsize}, rewardlabel/.style={pos=1, font=\scriptsize, xshift=3pt}, field/.style={-stealth', thin, flat-blue!20}, y=20pt, x=80pt]
    \node[] (anil) at (0, 0) {\begin{tikzpicture}
        \node[fill=gray!10, rounded corners=3pt, inner sep=5pt, minimum height=80pt, minimum width=180pt] (W) at (0, 0) {\begin{tikzpicture}
            \draw[ultra thick, -latex, shorten >=3pt] (0, 0) -- (0.4, 0.8);
            \draw[ultra thick, -latex, shorten >=3pt] (0.4, 0.8) -- (0.8, 0.8);
            \draw[ultra thick, -latex, shorten >=3pt] (0.8, 0.8) -- (1.2, 0.5);
            \draw[ultra thick, -latex, shorten >=3pt] (1.2, 0.5) -- (1.6, 0.9);
            \draw[ultra thick, -latex, shorten >=3pt] (1.6, 0.9) -- (2.0, 2);

            \node[circle, minimum size=5pt, fill=flat-red, draw=none, inner sep=0pt, label={[minimum size=0pt, font=\normalsize, inner sep=1pt, flat-red, fill=gray!10]300:$\bm{W}_{0}$}] at (0, 0) {};
            \node[circle, minimum size=5pt, fill=black, draw=none, inner sep=0pt] at (0.4, 0.8) {};
            \node[circle, minimum size=5pt, fill=black, draw=none, inner sep=0pt] at (0.8, 0.8) {};
            \node[circle, minimum size=5pt, fill=black, draw=none, inner sep=0pt] at (1.2, 0.5) {};
            \node[circle, minimum size=5pt, fill=black, draw=none, inner sep=0pt] at (1.6, 0.9) {};
            \node[shape=star, star points=5, star point ratio=2.25, rounded corners=0pt, fill=flat-green, minimum size=7pt, label={[minimum size=0pt, font=\normalsize, inner sep=1pt, flat-green, fill=gray!10, xshift=-8pt]180:$\bm{W}_{T}$}] at (2, 2) {};
        \end{tikzpicture}};
    \end{tikzpicture}};

    \node[right=110pt of anil.east] (comln) at (0, 0) {\begin{tikzpicture}
        \node[fill=flat-blue!10, rounded corners=3pt, inner sep=5pt, minimum height=80pt, minimum width=180pt] (W) at (0, 0) {\begin{tikzpicture}
            \draw[field] (0.0, -1) -- ++(50:9pt);
            \draw[field] (0.2, -1) -- ++(50:9pt);
            \draw[field] (0.4, -1) -- ++(55:9pt);
            \draw[field] (0.6, -1) -- ++(55:9pt);
            \draw[field] (0.8, -1) -- ++(60:9pt);
            \draw[field] (1.0, -1) -- ++(60:9pt);
            \draw[field] (1.2, -1) -- ++(65:9pt);
            \draw[field] (1.4, -1) -- ++(65:9pt);
            \draw[field] (1.6, -1) -- ++(70:9pt);
            \draw[field] (1.8, -1) -- ++(70:9pt);
            \draw[field] (2.0, -1) -- ++(85:9pt);
            
            \draw[field] (0.0, -0.5) -- ++(30:9pt);
            \draw[field] (0.2, -0.5) -- ++(30:9pt);
            \draw[field] (0.4, -0.5) -- ++(35:9pt);
            \draw[field] (0.6, -0.5) -- ++(35:9pt);
            \draw[field] (0.8, -0.5) -- ++(35:9pt);
            \draw[field] (1.0, -0.5) -- ++(40:9pt);
            \draw[field] (1.2, -0.5) -- ++(45:9pt);
            \draw[field] (1.4, -0.5) -- ++(50:9pt);
            \draw[field] (1.6, -0.5) -- ++(60:9pt);
            \draw[field] (1.8, -0.5) -- ++(65:9pt);
            \draw[field] (2.0, -0.5) -- ++(85:9pt);
            
            \draw[field] (0.0, 0) -- ++(45:9pt);
            \draw[field] (0.2, 0) -- ++(40:9pt);
            \draw[field] (0.4, 0) -- ++(40:9pt);
            \draw[field] (0.6, 0) -- ++(35:9pt);
            \draw[field] (0.8, 0) -- ++(30:9pt);
            \draw[field] (1.0, 0) -- ++(25:9pt);
            \draw[field] (1.2, 0) -- ++(35:9pt);
            \draw[field] (1.4, 0) -- ++(40:9pt);
            \draw[field] (1.6, 0) -- ++(45:9pt);
            \draw[field] (1.8, 0) -- ++(60:9pt);
            \draw[field] (2.0, 0) -- ++(80:9pt);
            
            \draw[field] (0.0, 0.5) -- ++(20:9pt);
            \draw[field] (0.2, 0.5) -- ++(30:9pt);
            \draw[field] (0.4, 0.5) -- ++(10:9pt);
            \draw[field] (0.6, 0.5) -- ++(5:9pt);
            \draw[field] (0.8, 0.5) -- ++(0:9pt);
            \draw[field] (1.0, 0.5) -- ++(-10:9pt);
            \draw[field] (1.2, 0.5) -- ++(-3:9pt);
            \draw[field] (1.4, 0.5) -- ++(5:9pt);
            \draw[field] (1.6, 0.5) -- ++(20:9pt);
            \draw[field] (1.8, 0.5) -- ++(50:9pt);
            \draw[field] (2.0, 0.5) -- ++(75:9pt);
            
            \draw[field] (0.0, 1.0) -- ++(-5:9pt);
            \draw[field] (0.2, 1.0) -- ++(-5:9pt);
            \draw[field] (0.4, 1.0) -- ++(-10:9pt);
            \draw[field] (0.6, 1.0) -- ++(-10:9pt);
            \draw[field] (0.8, 1.0) -- ++(-15:9pt);
            \draw[field] (1.0, 1.0) -- ++(-20:9pt);
            \draw[field] (1.2, 1.0) -- ++(-20:9pt);
            \draw[field] (1.4, 1.0) -- ++(-10:9pt);
            \draw[field] (1.6, 1.0) -- ++(20:9pt);
            \draw[field] (1.8, 1.0) -- ++(50:9pt);
            \draw[field] (2.0, 1.0) -- ++(60:9pt);
            
            \draw[field] (0.0, 1.5) -- ++(-20:9pt);
            \draw[field] (0.2, 1.5) -- ++(-20:9pt);
            \draw[field] (0.4, 1.5) -- ++(-30:9pt);
            \draw[field] (0.6, 1.5) -- ++(-30:9pt);
            \draw[field] (0.8, 1.5) -- ++(-30:9pt);
            \draw[field] (1.0, 1.5) -- ++(-30:9pt);
            \draw[field] (1.2, 1.5) -- ++(-20:9pt);
            \draw[field] (1.4, 1.5) -- ++(-10:9pt);
            \draw[field] (1.6, 1.5) -- ++(10:9pt);
            \draw[field] (1.8, 1.5) -- ++(40:9pt);
            \draw[field] (2.0, 1.5) -- ++(50:9pt);
            
            \draw[field] (0.0, 2.0) -- ++(-40:9pt);
            \draw[field] (0.2, 2.0) -- ++(-40:9pt);
            \draw[field] (0.4, 2.0) -- ++(-40:9pt);
            \draw[field] (0.6, 2.0) -- ++(-40:9pt);
            \draw[field] (0.8, 2.0) -- ++(-40:9pt);
            \draw[field] (1.0, 2.0) -- ++(-40:9pt);
            \draw[field] (1.2, 2.0) -- ++(-30:9pt);
            \draw[field] (1.4, 2.0) -- ++(-20:9pt);
            \draw[field] (1.6, 2.0) -- ++(-20:9pt);
            \draw[field] (1.8, 2.0) -- ++(10:9pt);

            \node[shape=star, star points=5, star point ratio=2.25, rounded corners=0pt, fill=flat-green, minimum size=7pt, label={[minimum size=0pt, font=\normalsize, inner sep=1pt, flat-green, fill=flat-blue!10, xshift=-8pt]180:$\bm{W}(T)$}] at (2, 2) {};
            \draw[ultra thick, -latex, shorten >=3pt] (0, 0) .. controls (0.5, 2) and (1.5, -1) .. (2, 2);
            
            \node[circle, minimum size=5pt, fill=flat-red, draw=none, inner sep=0pt, label={[minimum size=0pt, font=\normalsize, inner sep=1pt, flat-red, fill=flat-blue!10]300:$\bm{W}_{0}$}, anchor=center] at (0, 0) {};
        \end{tikzpicture}};
    \end{tikzpicture}};
    \node[below=0.5em of comln] (comln_label) {$\dfrac{d\bm{W}}{dt} = -\nabla \mathcal{L}\big(\bm{W}(t)\big)$};
    \node[] (anil_label) at (anil |- comln_label) {$\bm{W}_{t+1} = \bm{W}_{t} - \alpha \nabla \mathcal{L}(\bm{W}_{t})$};
    
    \node[above=0.5em of anil] (anil_title) {(a) Gradient-based Meta-Learning};
    \node[] (comln_title) at (anil_title -| comln) {(b) COMLN};
\end{tikzpicture}
    \caption{Illustration of the adaptation process in (a) a gradient-based meta-learning algorithm, such as ANIL \citep{raghu2019anil}, where the adapted parameters $\mW_{T}$ are given after $T$ steps of gradient descent, and in (b) Continuous-Time Meta-Learning (COMLN), where the adapted parameters $\mW(T)$ are the result of following the dynamics of a gradient vector field up to time $T$.}
    \label{fig:comln}
\end{figure}

\subsection{Gradient-based meta-learning}
\label{sec:gradient-based-meta-learning}

Gradient-based meta-learning methods aim to learn an initialization such that the model is able to adapt to a new task via gradient descent. Such methods are often cast as a bi-level optimization process: adapting the task-specific parameters $\vtheta$ in the inner loop, and training the (task-agnostic) meta-parameters $\Phi$ and initialization $\vtheta_{0}$ in the outer loop. The meta-learning objective is:
\begin{align}
    \min_{\vtheta_{0}, \Phi}\ \ & \E_{\tau}\big[\gL(\vtheta^{\tau}_{T}, \Phi; \Dtest)\big] \label{eq:gradient-based-meta-learning-objective}\\
    \textrm{s.t.}\ \ & \vtheta^{\tau}_{t+1} = \vtheta^{\tau}_{t} - \alpha\nabla_{\vtheta} \gL(\vtheta^{\tau}_{t}, \Phi; \Dtrain) & \vtheta_{0}^{\tau} &= \vtheta_{0} & & \forall \tau \sim p(\tau), \label{eq:gradient-based-meta-learning-adaptation}
\end{align}
where $T$ is the number of inner loop updates. For example, in the case of MAML \citep{finn17maml}, there is no additional meta-parameter other than the initialization ($\Phi \equiv \emptyset$); in ANIL \citep{raghu2019anil}, $\vtheta$ are the parameters of the last layer, and $\Phi$ are the parameters of the shared embedding network; in CAVIA \citep{zintgraf2019fast}, $\vtheta$ are referred to as context parameters.

During meta-training, the model is trained over many tasks $\tau$. The task-specific parameters $\vtheta$ are learned via gradient descent on $\Dtrain$.
The meta-parameters are then updated by evaluating the error of the trained model on the test dataset $\Dtest$. At meta-testing time, the meta-trained model is adapted on $\Dtrain$---i.e. applying \cref{eq:gradient-based-meta-learning-adaptation} with the learned meta-parameters $\vtheta_{0}$ and $\Phi$.

\subsection{Local sensitivity analysis of Ordinary Differential Equations}
\label{sec:local-sensitivity-analysis-odes}
Consider the following (autonomous) Ordinary Differential Equation (ODE):
\begin{align}
    \frac{d\vz}{dt} &= g\big(\vz(t); \vtheta\big) & \vz(0) &= \vz_{0}(\vtheta),
    \label{eq:ode}
\end{align}
where the dynamics $g$ and initial value $\vz_{0}$ may depend on some external parameters $\vtheta$, and integration is carried out from $0$ to some time $T$. \emph{Local sensitivity analysis} is the study of how the solution of this dynamical system responds to local changes in $\vtheta$; this effectively corresponds to calculating the derivative $d\vz(t)/d\vtheta$. We present here two methods to compute this derivative, with a special focus on their memory efficiency.

\paragraph{Adjoint sensitivity method} Based on the adjoint state \citep{pontryagin2018mathematical}, and taking its root in control theory \citep{lions2012non}, the \emph{adjoint sensitivity method} \citep{BrysonHo69,chavent1974identification} provides an efficient approach for evaluating derivatives of $\gL\big(\vz(T);\vtheta\big)$, a function of $\vz(T)$ the solution of the ODE in \cref{eq:ode}. This method, popularized lately by the literature on Neural ODEs \citep{chen2018neuralode}, requires the integration of the adjoint equation
\begin{align}
    \frac{d\va}{dt} &= -\va(t)\frac{\partial g\big(\vz(t);\vtheta\big)}{\partial \vz(t)} & \va(T) &= \frac{d\gL\big(\vz(T);\vtheta\big)}{d\vz(T)},
    \label{eq:adjoint-equation}
\end{align}
backward in time. The adjoint sensitivity method can be viewed as a continuous-time counterpart to backpropagation, where the forward pass would correspond to integrating \cref{eq:ode} forward in time from $0$ to $T$, and the backward pass to integrating \cref{eq:adjoint-equation} backward in time from $T$ to $0$.

One possible implementation, reminiscent of backpropagation through time (BPTT), is to store the intermediate values of $\vz(t)$ during the forward pass, and reuse them to compute the adjoint state during the backward pass. While several sophisticated checkpointing schemes have been proposed \citep{serban2003cvodes,DBLP:journals/corr/abs-1902-10298}, with different compute/memory trade-offs, the memory requirements of this approach typically grow with $T$; this is similar to the memory limitations standard gradient-based meta-learning methods suffer from as the number of gradient steps increases. An alternative is to augment the adjoint state $\va(t)$ with the original state $\vz(t)$, and to solve this augmented dynamical system backward in time \citep{chen2018neuralode}. This has the notable advantage that the memory requirements are now independent of $T$, since $\vz(t)$ are no longer stored during the forward pass, but they are recomputed on the fly during the backward pass.

\paragraph{Forward sensitivity method} While the adjoint method is related to reverse-mode automatic differentiation (backpropagation), the \emph{forward sensitivity method} \citep{feehery1997efficient,leis1988simultaneous,maly1996numerical,caracotsios1985sensitivity}, on the other hand, can be viewed as the continuous-time counterpart to forward (tangent-linear) mode differentiation \citep{Griewank2008}. This method is based on the fact that the derivative $\gS(t) \triangleq d\vz(t)/d\vtheta$ is the solution of the so-called \emph{forward sensitivity equation}

\begin{align}
    \frac{d\gS}{dt} &= \frac{\partial g\big(\vz(t); \vtheta\big)}{\partial \vz(t)}\gS(t) + \frac{\partial g\big(\vz(t);\vtheta\big)}{\partial \vtheta} & \gS(0) &= \frac{\partial \vz_{0}}{\partial \vtheta}.
    \label{eq:forward-sensitivity-equations}
\end{align}

This equation can be found throughout the literature in optimal control and system identification \citep{betts2010optimalcontrol,Biegler2010}. Unlike the adjoint method, which requires an explicit forward \emph{and} backward pass, the forward sensitivity method only requires the integration forward in time of the original ODE in \cref{eq:ode}, augmented by the sensitivity state $\gS(t)$ with the dynamics above. The memory requirements of the forward sensitivity method do not scale with $T$ either, but it now requires storing $\gS(t)$, which may be very large; we will come back to this problem in \cref{sec:challenges-optimization-meta-learning-objective}. We will simply note here that in discrete-time, this is the same issue afflicting forward-mode training of RNNs with real-time recurrent learning \citep[RTRL;][]{Williams1989}, or other meta-learning algorithms \citep{Sutton1992,franceschi2017forward,Xu2018}.

\section{Continuous-time adaptation}
\label{sec:continuous-time-adaptation}
In the limit of infinitely small steps, some optimization algorithms can be viewed as the solution trajectory of a differential equation. This point of view has often been taken to analyze their behavior \citep{platt88constrained,wilson2016lyapunov,su2014differential,orvieto2019shadowing}. In fact, some optimization algorithms such as gradient descent with momentum have even been introduced initially from the perspective of dynamical systems \citep{polyak1964heavyball}. As the simplest example, gradient descent with a constant step size $\alpha \rightarrow 0^{+}$ (i.e. $\alpha$ tends to $0$ by positive values) corresponds to following the dynamics of an autonomous ODE called a \emph{gradient vector field}
\begin{align}
    \vz_{t+1} &= \vz_{t} - \alpha \nabla f(\vz_{t}) & \underset{\alpha \rightarrow 0^{+}}{\longrightarrow} && \frac{d\vz}{dt} &= -\nabla f\big(\vz(t)\big),
    \label{eq:gradient-vector-field}
\end{align}
where the iterate $\vz(t)$ is now a continuous function of time $t$. The solution of this dynamical system is uniquely defined by the choice of the initial condition $\vz(0) = \vz_{0}$.

\subsection{Continuous-time meta-learning}
\label{sec:continuous-time-meta-learning}
In gradient-based meta-learning, the task-specific adaptation with gradient descent may also be replaced by a gradient vector field in the limit of infinitely small steps. Inspired by prior work in meta-learning \citep{raghu2019anil,javed2019oml}, we assume that an embedding network $f_{\Phi}$ with meta-parameters $\Phi$ is shared across tasks, and only the parameters $\mW$ of a task-specific linear classifier are adapted, starting at some initialization $\mW_{0}$. Instead of being the result of a few steps of gradient descent though, the final parameters $\mW(T)$ now correspond to integrating an ODE similar to \cref{eq:gradient-vector-field} up to a certain horizon $T$, with the initial conditions $\mW(0) = \mW_{0}$. We call this new meta-learning algorithm \underline{Co}ntinuous-Time \underline{M}eta-\underline{L}earni\underline{n}g\footnote{COMLN is pronounced \emph{chameleon}.} (COMLN).

Treating the learning algorithm as a continuous-time process has the notable advantage that the adapted parameter $\mW(T)$ is now differentiable wrt. the time horizon $T$ \citep[][Chap. 7]{wiggins1990dynamicalsystems}, in addition to being differentiable wrt. the initial conditions $\mW_{0}$---which plays a central role in gradient-based meta-learning, as described in \cref{sec:gradient-based-meta-learning}. This allows us to view $T$ as a meta-parameter on par with $\Phi$ and $\mW_{0}$, and to effectively optimize the amount of adaptation using stochastic gradient descent (SGD). The meta-learning objective of COMLN can be written as
\begin{align}
    \min_{\Phi, \mW_{0}, T}\ \ & \E_{\tau}\big[\gL\big(\mW_{\tau}(T);f_{\Phi}(\Dtest)\big)\big] \label{eq:comln-meta-learning-objective}\\
    \textrm{s.t.}\ \ & \frac{d\mW_{\tau}}{dt} = -\nabla \gL\big(\mW_{\tau}(t);f_{\Phi}(\Dtrain)\big) &\mW_{\tau}(0) &= \mW_{0} & \forall \tau \sim p(\tau),\label{eq:comln-meta-learning-adaptation}
\end{align}
where $f_{\Phi}(\Dtrain) = \{(f_{\Phi}(\vx_{m}), \vy_{m}) \mid (\vx_{m}, \vy_{m}) \in \Dtrain\}$ is the embedded training dataset, and $f_{\Phi}(\Dtest)$ is defined similarly for $\Dtest$. In practice, adaptation is implemented using a numerical integration scheme based on an iterative discretization of the problem, such as Runge-Kutta methods. Although a complete discussion of numerical solvers is outside of the scope of this paper, we recommend \citep{butcher2008numerical} for a comprehensive overview of numerical methods for solving ODEs.

\subsection{The challenges of optimizing the meta-learning objective}
\label{sec:challenges-optimization-meta-learning-objective}
\begin{figure}[t]
    \centering
    \begin{minipage}[c]{0.35\linewidth}
        \begin{tikzpicture}[every node/.style={inner sep=0pt}, Wlabel/.style={font=\scriptsize, fill=white}]
            \node[] at (0, 0) {\includegraphics[width=\linewidth]{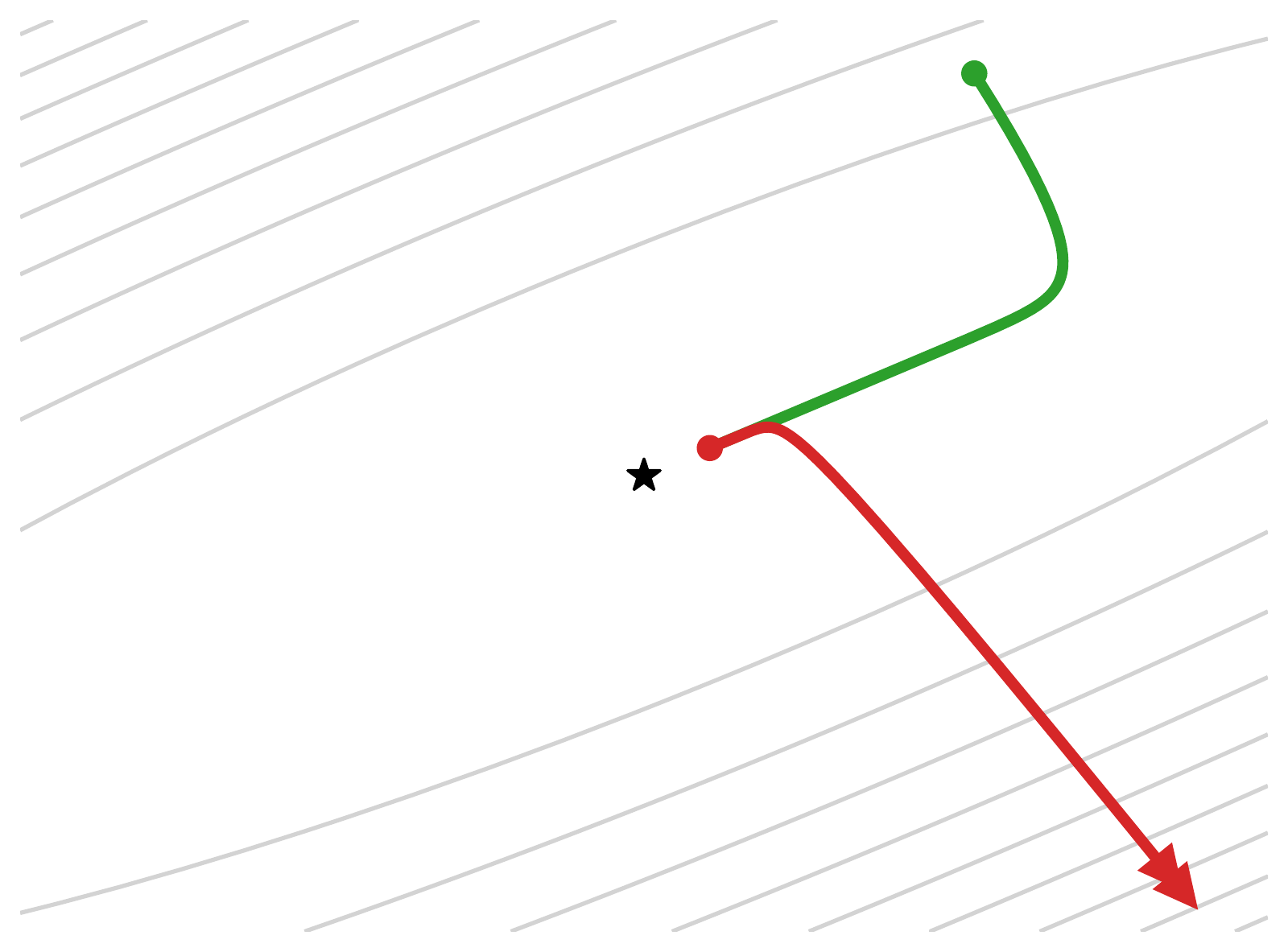}};
            \node[Wlabel, inner sep=2pt] at (0.8, 1.3) {$\mW(0)$};
            \node[Wlabel] at (-0.1, 0.4) {$\mW(T)$};
            \node[Wlabel] at (0, -0.3) {$\mW^{\star}$};
            \node[anchor=south west, draw=gray!20, fill=white, thick, inner sep=5pt, rounded corners=2pt] at (-2.2, -1.6) {\begin{tikzpicture}[y=10pt]
                \draw[-, draw=C2, very thick, font=\scriptsize] (0, 0) -- ++(0:1.5em) node[pos=1, anchor=west, xshift=5pt] {Forward};
                \draw[-, draw=C3, very thick, font=\scriptsize] (0, -1) -- ++(0:1.5em) node[pos=1, anchor=west, xshift=5pt] {Backward};
            \end{tikzpicture}};
        \end{tikzpicture}
    \end{minipage}\hspace{2ex}
    \begin{minipage}[c]{0.6\linewidth}
    \caption{Numerical instability of the adjoint method applied to the gradient vector field of a quadratic loss function. The trajectory in green starting at $\mW(0)$ corresponds to the integration of the dynamical system in \cref{eq:comln-meta-learning-adaptation} forward in time up to $T$, and the trajectory in red starting at $\mW(T)$ corresponds to its integration backward in time. Note that $T$ was chosen so that $\mW(T)$ does not reach the equilibrium/minimum of the loss $\mW^{\star}$.}
    \label{fig:adjoint_unstable}
    \end{minipage}
\end{figure}
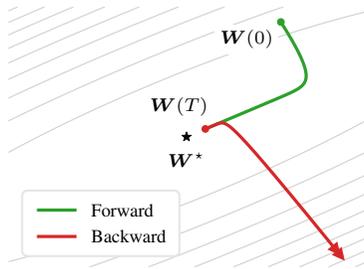

In order to minimize the meta-learning objective of COMLN, it is common practice to use (stochastic) gradient methods; that requires computing its derivatives wrt. the meta-parameters, which we call \emph{meta-gradients}. Our primary goal is to devise an algorithm whose memory requirements do not scale with the amount of adaptation $T$; this would contrast with standard gradient-based meta-learning methods that backpropagate through a sequence of gradient steps (similar to BPTT), where the intermediate parameters are stored during adaptation (i.e. $\vtheta^{\tau}_{t}$ for all $t$ in \cref{eq:gradient-based-meta-learning-adaptation}). Since this objective involves the solution $\mW(T)$ of an ODE, we can use either the adjoint method, or the forward sensitivity method, in order to compute the derivatives wrt. $\Phi$ and $\mW_{0}$ (see \cref{sec:local-sensitivity-analysis-odes}).

Although the adjoint method has proven to be an effective strategy for learning Neural ODEs, in practice computing the state $\mW(t)$ backward in time is numerically unstable when applied to a gradient vector field like the one in \cref{eq:comln-meta-learning-adaptation}, even for convex loss functions. \cref{fig:adjoint_unstable} shows an example where the trajectory of $\mW(t)$ recomputed backward in time (in red) diverges significantly from the original trajectory (in green) on a quadratic loss function, even though the two should match exactly in theory since they follow the same dynamics. Intuitively, recomputing $\mW(t)$ backward in time for a gradient vector field requires doing gradient \emph{ascent} on the loss function, which is prone to compounding numerical errors; this is closely related to the loss of entropy observed by \citet{maclaurin2015gradient}. This divergence makes the backward trajectory of $\mW(t)$ unreliable to find the adjoint state, ruling out the adjoint sensitivity method for computing the meta-gradients in COMLN.

The forward sensitivity method addresses this shortcoming by avoiding the backward pass altogether. However, it can also be particularly expensive here in terms of memory requirements, since the sensitivity state $\gS(t)$ in \cref{sec:local-sensitivity-analysis-odes} now corresponds to Jacobian matrices, such as $d\mW(t)/d\mW_{0}$. As the size $d$ of the feature vectors returned by $f_{\Phi}$ may be very large, this $Nd \times Nd$ Jacobian matrix would be almost impossible to store in practice; for example in our experiments, it can be as large as $d = 16,\!000$ for a ResNet-12 backbone. In \cref{sec:meta-gradient}, we will show how to apply forward sensitivity in a memory-efficient way, by leveraging the structure of the loss function. This is achieved by carefully decomposing the Jacobian matrices into smaller pieces that follow specific dynamics. We show in \cref{app:proofs-stability} that unlike the adjoint method, this process is stable.

\subsection{Connection with Almost No Inner-Loop (ANIL)}
\label{sec:connection-anil}
Similarly to ANIL \citep{raghu2019anil}, COMLN only adapts the parameters $\mW$ of the last linear layer of the neural network. There is a deeper connection between both algorithms though: while our description of the adaptation in COMLN (Eq. \ref{eq:comln-meta-learning-adaptation}) was independent of the choice of the ODE solver used to find the solution $\mW(T)$ in practice, if we choose an explicit Euler scheme (\citealp{euler1913eulermethod}; roughly speaking, discretizing \cref{eq:gradient-vector-field} from right to left), then the adaptation of COMLN becomes functionally equivalent to ANIL. However, this equivalence can greatly benefit from the memory-efficient algorithm to compute the meta-gradients described in \cref{sec:memory-efficient-meta-gradients}, based on the forward sensitivity method. This means that using the methods devised here for COMLN, we can effectively compute the meta-gradients of ANIL with a constant memory cost wrt. the number of gradient steps of adaptation, instead of relying on backpropagation (see also \cref{sec:memory-efficiency}).

\section{Memory-efficient meta-gradients}
\label{sec:memory-efficient-meta-gradients}
For some fixed task $\tau$ and $(\vx_{m}, \vy_{m}) \in \Dtrain$, let $\vphi_{m} = f_{\Phi}(\vx_{m}) \in \sR^{d}$ be the embedding of input $\vx_{m}$ through the feature extractor $f_{\Phi}$. Since we are confronted with a classification problem, the loss function of choice $\gL(\mW)$ is typically the cross-entropy loss. \citet{bohning1992multinomiallogreg} showed that the gradient of the cross-entropy loss wrt. $\mW$ can be written as
\begin{align}
    \nabla\gL\big(\mW;f_{\Phi}(\Dtrain)\big) &= \frac{1}{M}\sum_{m=1}^{M} (\vp_{m} - \vy_{m})\vphi_{m}^{\top},
    \label{eq:decomposition-gradient}
\end{align}
where $\vp_{m} = \softmax(\mW\vphi_{m})$ is the vector of probabilities returned by the neural network. The key observation here is that the gradient can be decomposed as a sum of $M$ rank-one matrices, where the feature vectors $\vphi_{m}$ are independent of $\mW$. Therefore we can fully characterize the gradient of the cross-entropy loss with $M$ vectors $\vp_{m} - \vy_{m} \in \sR^{N}$, as opposed to the full $N\times d$ matrix. This is particularly useful in the context of few-shot classification, where the number of training examples $M$ is small, and typically significantly smaller than the embedding size $d$.

\subsection{Decomposition of the meta-gradients}
\label{sec:meta-gradient}
We saw in \cref{sec:challenges-optimization-meta-learning-objective} that the forward sensitivity method was the only stable option to compute the meta-gradients of COMLN. However, naively applying the forward sensitivity equation would involve quantities that typically scale with $d^{2}$, which can be too expensive in practice. Using the structure of \cref{eq:decomposition-gradient}, the Jacobian matrices appearing in the computation of the meta-gradients for COMLN can be decomposed in such a way that only small quantities will depend on time.

\paragraph{Meta-gradients wrt. $\mW_{0}$} By the chain rule of derivatives, it is sufficient to compute the Jacobian matrix $d\mW(T)/d\mW_{0}$ in order to obtain the meta-gradient wrt. $\mW_{0}$. We show in \cref{app:proof-W0} that the sensitivity state $d\mW(t)/d\mW_{0}$ can be decomposed as:
\begin{equation}
    \frac{d\mW(t)}{d\mW_{0}} = \mI - \sum_{i=1}^{M}\sum_{j=1}^{M}\mB_{t}[i,j]\otimes \vphi_{i}\vphi_{j}^{\top},
    \label{eq:jacobian-W0}
\end{equation}
where $\otimes$ is the Kronecker product, and each $\mB_{t}[i, j]$ is an $N \times N$ matrix, solution of the following system of ODEs\footnote{Here we used the notation $\mB_{t}[i, j]$ to make the dependence on $t$ explicit, without overloading the notation. A more precise notation would be $\mB[i,j](t)$.}
\begin{align}
    \frac{d\mB_{t}[i, j]}{dt} &= \mathbbm{1}(i = j)\mA_{i}(t) - \mA_{i}(t)\sum_{m=1}^{M}(\vphi_{i}^{\top}\vphi_{m})\mB_{t}[m,j] & \mB_{0}[i, j] &= \vzero,
    \label{eq:dynamical-system-B}
\end{align}
and $\mA_{i}(t)$, defined in \cref{app:proof-W0}, is also an $N \times N$ matrix that only depends on $\mW(t)$ and $\vphi_{i}$. The main consequence of this decomposition is that we can simply integrate the augmented ODE in $\big\{\mW(t), \mB_{t}[i, j]\big\}$ up to $T$ to obtain the desired Jacobian matrix, along with the adapted parameters $\mW(T)$. Furthermore, in contrast to naively applying the forward sensitivity method (see \cref{sec:challenges-optimization-meta-learning-objective}), the $M^{2}$ matrices $\mB_{t}[i, j]$ are significantly smaller than the full Jacobian matrix. In fact, we show in \cref{app:projection} that we can compute vector-Jacobian products---required for the chain rule---using only these smaller matrices, and without ever having to explicitly construct the full $Nd \times Nd$ Jacobian matrix $d\mW(t)/d\mW_{0}$ with \cref{eq:jacobian-W0}.

\paragraph{Meta-gradients wrt. $\Phi$} To backpropagate the error through the embedding network $f_{\Phi}$, we need to first compute the gradients of the outer-loss wrt. the feature vectors $\vphi_{m}$. Again, by the chain rule, we can get these gradients with the Jacobian matrices $d\mW(T)/d\vphi_{m}$. Similar to \cref{eq:jacobian-W0}, we can show that these Jacobian matrices can be decomposed as:
\begin{equation}
    \frac{d\mW(t)}{d\vphi_{m}} = -\bigg[\vs_{m}(t) \otimes \mI + \sum_{i=1}^{M}\mB_{t}[i, m]\mW_{0}\otimes \vphi_{i} + \sum_{i=1}^{M}\sum_{j=1}^{M}\vz_{t}[i, j, m]\vphi_{j}^{\top}\otimes \vphi_{i}\bigg],
    \label{eq:jacobian-phi}
\end{equation}
where $\vs_{m}(t)$ and $\vz_{t}[i, j, m]$ are vectors of size $N$, that follow some dynamics; the exact form of this system of ODEs, as well as the proof of this decomposition, are given in \cref{app:proof-phi}. Crucially, the only quantities that depend on time are small objects independent of the embedding size $d$. Following the same strategy as above, we can incorporate these vectors in the augmented ODE, and integrate it to get the necessary Jacobians. Once all the $d\mW(t)/d\vphi_{m}$ are known, for all the training datapoints, we can apply standard backpropagation through $f_{\Phi}$ to obtain the meta-gradients wrt. $\Phi$.

\paragraph{Meta-gradient wrt. $T$} One of the major novelties of COMLN is the capacity to meta-learn the amount of adaptation using stochastic gradient descent. To compute the meta-gradient wrt. the time horizon $T$, we can directly borrow the results derived by \citet{chen2018neuralode} in the context of Neural ODEs, and apply it to our gradient vector field in \cref{eq:comln-meta-learning-adaptation} responsible for adaptation:
\begin{equation}
    \frac{d\gL\big(\mW(T);f_{\Phi}(\Dtest)\big)}{dT} = -\bigg[\frac{\partial \gL\big(\mW(T);f_{\Phi}(\Dtest)\big)}{\partial \mW(T)}\bigg]^{\top}\frac{\partial \gL\big(\mW(T);f_{\Phi}(\Dtrain)\big)}{\partial \mW(T)}.
\end{equation}
The proof is available in \cref{app:proof-T}. Interestingly, we find that this involves the alignment between the vectors of partial derivatives of the inner-loss and the outer-loss at $\mW(T)$, which appeared in various contexts in the meta-learning and the multi-task learning literature \citep{li2018learning,rothfuss2019promp,yu2020gradientsrugery,von2021sparsemaml}.

\subsection{Memory efficiency}
\label{sec:memory-efficiency}
Although naively applying the forward sensitivity method would be memory intensive, we have shown in \cref{sec:meta-gradient} that the Jacobians can be carefully decomposed into smaller pieces. It turns out that even the parameters $\mW(t)$ can be expressed using the vectors $\vs_{m}(t)$ from the decomposition in \cref{eq:jacobian-phi}; see \cref{app:proof-W} for details. As a consequence, to compute the adapted parameters $\mW(T)$ as well as all the necessary meta-gradients, it is sufficient to integrate a dynamical system in $\big\{\mB_{t}[i, j], \vs_{m}(t), \vz_{t}[i, j, m]\big\}$ (see \cref{alg:comln-meta-training,alg:comln-dynamics-meta-training} in App. \ref{app:meta-training}), involving exclusively quantities that are independent of the embedding size $d$. Instead, the size of that system scales with $M$ the total number of training examples, which is typically much smaller than $d$ for few-shot classification.

\begin{table}[t]
    \caption{Memory required to compute meta-gradients for different algorithms. Exact: the method returns the exact meta-gradients. Full net.: the whole network is adapted, with a number of meta-parameters $|\vtheta|$. The requirements for checkpointing are taken from \citep{shaban2019truncated}. Note that typically $M \ll d$ in few-shot learning.}
    \label{tab:memory_usage}
    \centering
    \begin{tabular}{lccc}
        \toprule
        \textbf{Model} & \textbf{Exact} & \textbf{Full net.} & \textbf{Memory}\\
        \midrule
        MAML \citep{finn17maml} & \cmark & \cmark & $\gO(|\vtheta| \cdot T)$ \\
        ANIL \citep{raghu2019anil} & \cmark & \xmark & $\gO(Nd \cdot T)$ \\
        Checkpointing (every $\sqrt{T}$ steps) & \cmark & \cmark & $\gO(|\vtheta| \cdot \sqrt{T})$ \\
        iMAML \citep{rajeswaran2019imaml} & \xmark & \cmark & $\gO(|\vtheta|)$ \\
        \midrule
        Forward sensitivity (naive) & \cmark & \xmark & $\gO(N^{2}d^{2} + MNd^{2})$\\
        COMLN & \cmark & \xmark & $\gO(M^{2}N^{2} + M^{3}N)$ \\
        \bottomrule
    \end{tabular}
\end{table}

\cref{tab:memory_usage} shows a comparison of the memory cost for different algorithms. It is important to note that contrary to other standard gradient-based meta-learning methods, the memory requirements of COMLN do not scale with the amount of adaptation $T$ (i.e. the number of gradient steps in MAML \& ANIL), while still returning the exact meta-gradients---unlike iMAML \citep{rajeswaran2019imaml}, which only returns an approximation of the meta-gradients. We verified empirically this efficiency, both in terms of memory and computation costs, in \cref{sec:empirical-efficiency}.

\section{Experiments}
\label{sec:experiments}
For our embedding network $f_{\Phi}$, we consider two commonly used architectures in meta-learning: Conv-4, a convolutional neural network with 4 convolutional blocks, and ResNet-12, a 12-layer residual network \citep{he2016resnet}. Note that following \citet{lee2019metaoptnet}, ResNet-12 does not include a global pooling layer at the end of the network, leading to feature vectors with embedding dimension $d = 16,\!000$. Additional details about these architectures are given in \cref{app:experiment-details}. To compute the adapted parameters and the meta-gradients in COMLN, we integrate the dynamical system described in \cref{sec:memory-efficiency} with a 4th order Runge-Kutta method with a Dormand Prince adaptive step size \citep{runge1895numerische,dormand1980family}; we will come back to the choice of this numerical solver in \cref{sec:empirical-efficiency}. Furthermore to ensure that $T > 0$, we parametrized it with an exponential activation.

\subsection{Few-shot image classification}
\label{sec:few-shot-image-classification}
We evaluate COMLN on two standard few-shot image classification benchmarks: the \textit{mini}ImageNet \citep{vinyals2016matching} and the \textit{tiered}ImageNet datasets \citep{ren2018meta}, both datasets being derived from ILSVRC-2012 \citep{russakovsky2015imagenet}. The process for creating tasks follows the standard procedure from the few-shot classification literature \citep{santoro2016mann}, with distinct classes between the different splits. \textit{mini}Imagenet consists of 100 classes, split into 64 training classes, 16 validation classes, and 20 test classes. \textit{tiered}ImageNet consists of 608 classes grouped into 34 high-level categories from ILSVRC-2012, split into 20 training, 6 validation, and 8 testing categories---corresponding to 351/97/160 classes respectively; \citet{ren2018meta} argue that separating data according to high-level categories results in a more difficult and more realistic regime.

\begin{table}[t]
\centering
\caption{Few-shot classification on \textit{mini}ImageNet \& \textit{tiered}ImageNet. The average accuracy (\%) on 1,000 held-out meta-test tasks is reported with $95\%$ confidence interval. \cmark \ denotes gradient-based meta-learning algorithms. $\star$ denotes baseline results we executed using the official implementations.}
\label{tab:results-architecture}
\begin{adjustbox}{center}
\small
\begin{tabular}{lclcccc}
\toprule
\multirow{2}{*}{\textbf{Model}} & & \multirow{2}{*}{\textbf{Backbone}} & \multicolumn{2}{c}{\textbf{\textit{mini}ImageNet 5-way}} & \multicolumn{2}{c}{\textbf{\textit{tiered}ImageNet 5-way}} \\ \cmidrule(lr){4-5} \cmidrule(lr){6-7}
 & & & \textbf{1-shot} & \textbf{5-shot} & \textbf{1-shot} & \textbf{5-shot} \\ \midrule
MAML \citep{finn17maml} & \cmark & Conv-4 & $48.70 \pm 1.84^{\phantom{\star}}$ & $63.11 \pm 0.92^{\phantom{\star}}$ & $51.67 \pm 1.81^{\phantom{\star}}$ & $70.30 \pm 1.75^{\phantom{\star}}$ \\
ANIL \citep{raghu2019anil} & \cmark & Conv-4 & $46.30 \pm 0.40^{\phantom{\star}}$ & $61.00 \pm 0.60^{\phantom{\star}}$ & $49.35 \pm 0.26^{\phantom{\star}}$ & $65.82 \pm 0.12^{\phantom{\star}}$ \\
Meta-SGD \citep{li2017metasgd} & \cmark & Conv-4 & $50.47 \pm 1.87^{\phantom{\star}}$ & $64.03 \pm 0.94^{\phantom{\star}}$ & $52.80 \pm 0.44^{\phantom{\star}}$ & $62.35 \pm 0.26^{\phantom{\star}}$ \\
CAVIA \citep{zintgraf2019fast} & \cmark & Conv-4 & $51.82 \pm 0.65^{\phantom{\star}}$ & $65.85 \pm 0.55^{\phantom{\star}}$ & $52.41 \pm 2.64^{\star}$ & $67.55 \pm 2.05^{\star}$ \\
iMAML \citep{rajeswaran2019imaml} & \cmark & Conv-4 & $49.30 \pm 1.88^{\phantom{\star}}$ & $59.77 \pm 0.73^{\star}$ & $38.54 \pm 1.37^{\star}$ & $60.24 \pm 0.76^{\star}$ \\
MetaOptNet-RR \citep{lee2019metaoptnet} &  & Conv-4 & $\mathbf{53.23 \pm 0.59^{\phantom{\star}}}$ & $69.51 \pm 0.48^{\phantom{\star}}$ & $54.63 \pm 0.67^{\phantom{\star}}$ & $\mathbf{72.11 \pm 0.59^{\phantom{\star}}}$ \\
MetaOptNet-SVM \citep{lee2019metaoptnet} &  & Conv-4 & $52.87 \pm 0.57^{\phantom{\star}}$ & $68.76 \pm 0.48^{\phantom{\star}}$ & $\mathbf{54.71 \pm 0.67^{\phantom{\star}}}$ & $71.79 \pm 0.59^{\phantom{\star}}$ \\ \midrule
\textbf{COMLN (Ours)} & \cmark & Conv-4 & $53.01 \pm 0.62^{\phantom{\star}}$ & $\mathbf{70.54 \pm 0.54^{\phantom{\star}}}$ & $54.30 \pm 0.69^{\phantom{\star}}$ & $71.35 \pm 0.57^{\phantom{\star}}$\\ \midrule
MAML \citep{finn17maml} & \cmark & ResNet-12 & $49.92 \pm 0.65^{\phantom{\star}}$ & $63.93 \pm 0.59^{\phantom{\star}}$ & $55.37 \pm 0.74^{\phantom{\star}}$ & $72.93 \pm 0.60^{\phantom{\star}}$\\
ANIL \citep{raghu2019anil} & \cmark & ResNet-12 & $49.65 \pm 0.65^{\phantom{\star}}$ & $59.51 \pm 0.56^{\phantom{\star}}$ & $54.77 \pm 0.76^{\phantom{\star}}$ & $69.28 \pm 0.67^{\phantom{\star}}$\\
MetaOptNet-RR \citep{lee2019metaoptnet} &  & ResNet-12 & $61.41 \pm 0.61^{\phantom{\star}}$ & $77.88 \pm 0.46^{\phantom{\star}}$ & $65.36 \pm 0.71^{\phantom{\star}}$ & $81.34 \pm 0.52^{\phantom{\star}}$ \\
MetaOptNet-SVM \citep{lee2019metaoptnet} &  & ResNet-12 & $\mathbf{62.64 \pm 0.61^{\phantom{\star}}}$ & $\mathbf{78.63 \pm 0.46^{\phantom{\star}}}$ & $\mathbf{65.99 \pm 0.72^{\phantom{\star}}}$ & $\mathbf{81.56 \pm 0.53^{\phantom{\star}}}$\\ \midrule
\textbf{COMLN (Ours)} & \cmark & ResNet-12 & $59.26 \pm
0.65^{\phantom{\star}}$ & $77.26 \pm 0.49^{\phantom{\star}}$ & $62.93 \pm 0.71^{\phantom{\star}}$ & $81.13 \pm 0.53^{\phantom{\star}}$ \\ \bottomrule
\end{tabular}
\end{adjustbox}
\end{table}

\cref{tab:results-architecture} shows the average accuracies of COMLN compared to various meta-learning methods, be it gradient-based or not. For both backbones, COMLN decisively outperforms all other gradient-based meta-learning methods. Compared to methods that explicitly backpropagate through the learning process, such as MAML or ANIL, the performance gain shown by COMLN could be credited to the longer adaptation $T$ it learns, as opposed to a small number of gradient steps---usually about $10$ steps; this was fully enabled by our memory-efficient method to compute meta-gradients, which does not scale with the length of adaptation anymore (see \cref{sec:memory-efficiency}). We analyse the evolution of $T$ during meta-training for these different settings in \cref{app:analysis-learned-horizon}. In almost all settings, COMLN is even closing the gap with a strong non-gradient-based method like MetaOptNet; the remainder may be explained in part by the training choices made by \citet{lee2019metaoptnet} (see \cref{app:experiment-details} for details).

\subsection{Empirical efficiency of COMLN}
\label{sec:empirical-efficiency}
In \cref{sec:memory-efficiency}, we showed that our algorithm to compute the meta-gradients, based on forward differentiation, had a memory cost independent of the length of adaptation $T$. We verify this empirically in \cref{fig:conv4_memory_runtime}, where we compare the memory required by COMLN and other methods to compute the meta-gradients on a single task, with a Conv-4 backbone (\cref{fig:resnet12_memory_runtime} in \cref{app:details-measuring-efficiency} shows similar results for ResNet-12). To ensure an aligned comparison between discrete and continuous time, we use a conversion corresponding to a learning rate $\alpha = 0.01$ in \cref{eq:gradient-based-meta-learning-adaptation}; see \cref{app:details-measuring-efficiency} for a justification. As expected, the memory cost increases for both MAML and ANIL as the number of gradient steps increases, while it remains constant for iMAML and COMLN. Interestingly, we observe that the cost of COMLN is equivalent to the cost of running ANIL for a small number of steps, showing that the additional cost of integrating the augmented ODE in \cref{sec:memory-efficiency} to compute the meta-gradients is minimal.

Increasing the length of adaptation also has an impact on the time it takes to compute the adapted parameters, and the meta-gradients. \cref{fig:conv4_memory_runtime} (right) shows how the runtime increases with the amount of adaptation for different algorithms. We see that the efficiency of COMLN depends on the numerical solver used. When we use a simple explicit-Euler scheme, the time taken to compute the meta-gradients matches the one of ANIL; this behavior empirically confirms our observation in \cref{sec:connection-anil}. When we use an adaptive numerical solver, such as Runge-Kutta (RK) with a Dormand Prince step size, this computation can be significantly accelerated, thanks to the smaller number of function evaluations. In practice, we show in \cref{app:choice-numerical-solver} that the choice of the ODE solver has a very minimal impact on the accuracy.

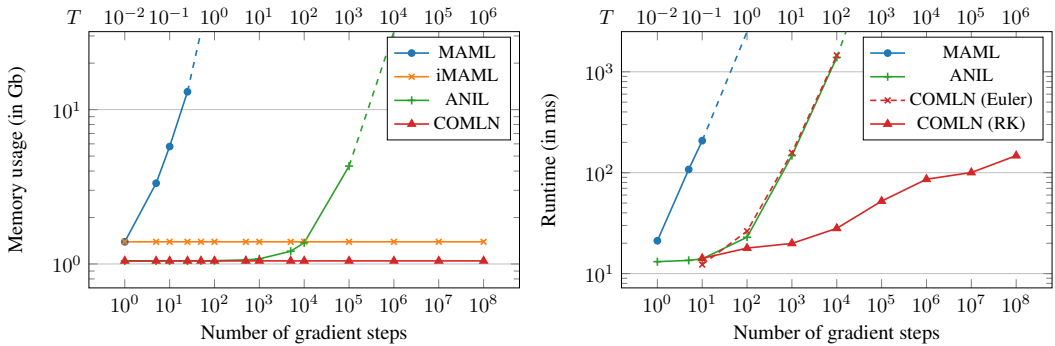
\begin{figure}[t]
    \centering
    \begin{tikzpicture}[]

\pgfkeys{
    /pgf/number format/.cd,
    sci,
    sci generic={mantissa sep=\times,exponent={10^{#1}}}
}

\begin{groupplot}[
    group style={group size=2 by 2, horizontal sep=5em}
]

\nextgroupplot[
    ymajorgrids,
    ymode=log,
    xmode=log,
    ylabel={Memory usage (in Gb)},
    xlabel={Number of gradient steps},
    ymax=32,
    xminorticks=false,
    ymin=0.7,
    ytick={1e0, 1e1},
    minor ytick={8e-1, 9e-1, 2e0, 3e0, 4e0, 5e0, 6e0, 7e0, 8e0, 9e0, 2e1, 3e1},
    height=6cm,
    width=9cm,
    legend style={font=\small},
]

\addplot[color=tableau10_C0, mark=*, mark size=1.5pt, thick] coordinates {
    (1, 1.392)
    (5, 3.335)
    (10, 5.765)
    (25, 13.052)
};

\addplot[color=tableau10_C1, mark=x, thick] coordinates {
    (1, 1.392)
    (5, 1.392)
    (10, 1.392)
    (25, 1.392)
    (50, 1.392)
    (100, 1.392)
    (500, 1.392)
    (1000, 1.392)
    (5000, 1.392)
    (10000, 1.392)
    (100000, 1.392)
    (1000000, 1.392)
    (10000000, 1.392)
    (100000000, 1.392)
};

\addplot[color=tableau10_C2, mark=+, thick] coordinates {
    (1, 1.046)
    (5, 1.046)
    (10, 1.047)
    (25, 1.047)
    (50, 1.048)
    (100, 1.050)
    (500, 1.063)
    (1000, 1.079)
    (5000, 1.209)
    (10000, 1.372)
    (100000, 4.308)
};

\addplot[color=tableau10_C3, mark=triangle*, thick] coordinates {
    (1, 1.047)
    (5, 1.047)
    (10, 1.047)
    (25, 1.047)
    (50, 1.047)
    (100, 1.047)
    (500, 1.047)
    (1000, 1.047)
    (5000, 1.047)
    (10000, 1.047)
    (100000, 1.047)
    (1000000, 1.047)
    (10000000, 1.047)
    (100000000, 1.047)
};

\addplot[color=tableau10_C0, dashed, thick] coordinates {(25, 13.052) (50, 32)};
\addplot[color=tableau10_C2, dashed, thick] coordinates {(100000, 4.308) (1000000, 32)};

\legend{MAML, iMAML, ANIL, COMLN}

\nextgroupplot[
    ymajorgrids,
    ymode=log,
    xmode=log,
    ylabel={Runtime (in ms)},
    xlabel={Number of gradient steps},
    xminorticks=false,
    height=6cm,
    width=9cm,
    legend style={font=\small},
    ymax=2.5e3
]

\addplot[color=tableau10_C0, mark=*, mark size=1.5pt, thick] coordinates {
    (1, 21.15)
    (5, 107.65)
    (10, 207.56)
};

\addplot[color=tableau10_C2, mark=+, thick] coordinates {
    (1, 13.13)
    (5, 13.52)
    (10, 13.99)
    (100, 22.95)
    (1000, 147.69)
    (10000, 1386.54)
};

\addplot[color=tableau10_C3, mark=x, densely dashed, thick, mark options={solid}] coordinates {
    (10, 12.28)
    (100, 26.37)
    (1000, 157.78)
    (10000, 1453.21)
};

\addplot[color=tableau10_C3, mark=triangle*, thick] coordinates {
    (10, 14.15)
    (100, 17.90)
    (1000, 19.91)
    (10000, 28.08)
    (100000, 52.29)
    (1000000, 86.14)
    (10000000, 100.38)
    (100000000, 147.38)
};

\addplot[color=tableau10_C0, dashed, thick] coordinates {(10, 207.56) (100, 2.5e3)};
\addplot[color=tableau10_C2, dashed, thick] coordinates {(10000, 1386.54) (16000, 2.5e3)};

\legend{MAML, ANIL, COMLN (Euler), COMLN (RK)}

\nextgroupplot[
  at=(group c1r1.north),
  axis y line=none,
  axis x line*=top,
  x axis line style=-,
  grid=none,
  xmode=log,
  xlabel={$T$},
  xmin=0.01,
  xmax=1000000,
  xminorticks=false,
  xtick={0.01, 0.1, 1, 10, 100, 1000, 10000, 100000, 1000000},
  height=2cm,
  width=9cm,
  enlarge x limits=true,
  every axis x label/.style={
    at={(ticklabel* cs:0.)},
    anchor=south east,
  },
  axis line style={draw=none},
  x tick style={draw=none},
  ymin=0,
  ymax=1
]

\nextgroupplot[
  at=(group c2r1.north),
  axis y line=none,
  axis x line*=top,
  x axis line style=-,
  grid=none,
  xmode=log,
  xlabel={$T$},
  xmin=0.01,
  xmax=1000000,
  xminorticks=false,
  xtick={0.01, 0.1, 1, 10, 100, 1000, 10000, 100000, 1000000},
  height=2cm,
  width=9cm,
  enlarge x limits=true,
  every axis x label/.style={
    at={(ticklabel* cs:0.)},
    anchor=south east,
  },
  axis line style={draw=none},
  x tick style={draw=none},
  ymin=0,
  ymax=1
]

\end{groupplot}

\end{tikzpicture}
    \caption{Empirical efficiency of COMLN on a single 5-shot 5-way task, with a Conv-4 backbone. (Left) Memory usage for computing the meta-gradients as a function of the number of inner-gradient steps. The extrapolated dashed lines correspond to the method reaching the memory capacity of a Tesla V100 GPU with 32Gb of memory. (Right) Average time taken (in ms) to compute the exact meta-gradients. The extrapolated dashed lines correspond to the method taking over 2 seconds.}
    \label{fig:conv4_memory_runtime}
    \vspace*{-1em}
\end{figure}

\section{Related work}
\label{sec:related-work}
We are interested in meta-learning \citep{bengio1991,schmidhuber1987evolutionary,thrun2012learningtolearn}, and in particular we focus on gradient-based meta-learning methods \citep{finn2018phd}, where the learning rule is based on gradient descent. While in MAML \citep{finn17maml} the whole network was updated during this process, follow-up works have shown that it is generally sufficient to share most parts of the neural network, and to only adapt a few layers \citep{raghu2019anil,chen2019modular,tian2020rethinking}. Even though this hypothesis has been challenged recently \citep{arnold2021embedding}, COMLN also updates only the last layer of a neural network, and therefore can be viewed as a continuous-time extension of ANIL \citep{raghu2019anil}; see also \cref{sec:connection-anil}. With its shared embedding network across tasks, COMLN is also connected to metric-based meta-learning methods \citep{vinyals2016matching,snell2017protonet, sung2018learning,bertinetto2018r2d2,lee2019metaoptnet}.

\citet{zhang2021metanode} introduced a formulation where the adaptation of prototypes follows a gradient vector field, but finally opted for modeling it as a Neural ODE \citep{chen2018neuralode}. Concurrent to our work, \citet{li2021metalearningadjoint} also propose a formulation of adaptation based on a gradient vector field, and use the adjoint method to compute the meta-gradients, despite the challenges we identified in Sec.~\ref{sec:challenges-optimization-meta-learning-objective}; \citet{li2021metalearningadjoint} also acknowledge these challenges, and they limit their analysis to relatively small values of $T$ (in comparison to the ones learned by COMLN), hence further from the equilibrium, to circumvent this issue altogether. \citet{zhou2021metantk} also uses a gradient vector field to motivate a novel method with a closed-form adaptation; COMLN still explicitly updates the parameters following the gradient vector field, since there is no closed-form solution of \cref{eq:comln-meta-learning-adaptation}. As mentioned in \cref{sec:continuous-time-adaptation}, treating optimization as a continuous-time process has been used to analyze the convergence of different optimization algorithms, including the meta-optimization of MAML \citep{xu2021metalearningcontinuous}, or to introduce new meta-optimizers based on different integration schemes \citep{im2019mamlrk}. \citet{guo2021personalized} also uses meta-learning to learn new integration schemes for ODEs. Although this is a growing literature at the intersection of meta-learning and dynamical systems, our work is the first algorithm that uses a gradient vector field for adaptation in meta-learning (see also \citet{li2021metalearningadjoint}).

Beyond the memory efficiency of our method, one of the main benefits of the continuous-time perspective is that COMLN is capable of learning when to stop the adaptation, as opposed to taking a number of gradient steps fixed ahead of time. However unlike \citet{chen2020learning2stop}, where the number of gradient steps are optimized (up to a maximal number) with variational methods, we incorporate the amount of adaptation as a (continuous) meta-parameter that can be learned using SGD. To compute the meta-gradients, which is known to be challenging for long sequences in gradient-based meta-learning, we use forward-mode differentiation as an alternative to backpropagation through the learning process, similar to prior work in meta-learning \citep{franceschi2017forward,jiwoong2021online} and hyperparameter optimization over long horizons \citep{micaelli2021gradient}. This yields the exact meta-gradients in constant memory, without any assumption on the optimality of the inner optimization problem, which is necessary when using the normal equations \citep{bertinetto2018r2d2}, or to apply implicit differentiation \citep{rajeswaran2019imaml}.

\section{Conclusion and Future Work}
\label{sec:conclusion}
In this paper, we have introduced \emph{Continuous-Time Meta-Learning} (COMLN), a novel algorithm that treats the adaptation in meta-learning as a continuous-time process, by following the dynamics of a gradient vector field up to a certain time horizon $T$. One of the major novelties of treating adaptation in continuous time is that the amount of adaptation $T$ is now a continuous quantity, that can be viewed as a meta-parameter and can be learned using SGD, alongside the initial conditions and the parameters of the embedding network. In order to learn these meta-parameters, we have also introduced a novel practical algorithm based on forward mode automatic differentiation, capable of efficiently computing the exact meta-gradients using an augmented dynamical system. We have verified empirically that this algorithm was able to compute the meta-gradients in constant memory, making it the first gradient-based meta-learning approach capable of computing the exact meta-gradients with long sequences of adaptation using gradient methods. In practice, we have shown that COMLN significantly outperforms other standard gradient-based meta-learning algorithms.

In addition to having a single meta-parameter $T$ that drives the adaptation of all possible tasks, the fact that the time horizon can be learned with SGD opens up new possibilities for gradient-based methods. For example, we could imagine treating $T$ not as a shared meta-parameters, but as a task-specific parameter. This would allow the learning process to be more \emph{adaptive}, possibly with different behaviors depending on the difficulty of the task. This is left as a future direction of research.

\section*{Acknowledgements}
The authors are grateful to Samsung Electronics Co., Ldt., CIFAR, and IVADO for their funding and Calcul Qu\'{e}bec and Compute Canada for providing us with the computing resources.

\newpage
\section*{Reproducibility statement}
\label{sec:reproducibility-statement}

We provide in \cref{app:meta-training} a full description in pseudo-code of the meta-training procedure (\cref{alg:comln-meta-training}), along with the exact dynamics of the ODE (\cref{alg:comln-dynamics-meta-training}) and the projection operations (\cref{alg:comln-projection-W0,alg:comln-projection-phi}) to avoid explicitly building the Jacobian matrices to compute Jacobian-vector products (see \cref{sec:meta-gradient}).

We also provide in \cref{app:source-code} a snippet of code in JAX \citep{jax2018github} to compute the adapted parameters $\mW(T)$, as well as all the necessary objects $\big\{\mB_{t}[i, j], \vs_{m}(t), \vz_{t}[i, j, m]\big\}$ to compute all the meta-gradients (see \cref{sec:memory-efficiency}). We also give in Code Snippet \ref{code:jax-snippet-meta-gradients} the code to compute the meta-gradients wrt. the initialization $\mW_{0}$ and the integration time $T$. Computing the meta-gradients wrt. $\Phi$ involves non-minimal dependencies on Haiku \citep{haiku2020github}, and therefore is omitted here. The full code is available at \href{https://github.com/tristandeleu/jax-comln}{\texttt{https://github.com/tristandeleu/jax-comln}}.

\paragraph{Data generation \& hyperparameters} We used the \textit{mini}ImageNet and \textit{tiered}ImageNet datasets provided by \citet{lee2019metaoptnet} in order to create the 1-shot 5-way and 5-shot 5-way tasks for both datasets. During evaluation, a fixed set of $1,\!000$ tasks was sampled for each setting; this means that both architectures for COMLN have been evaluated using exactly the same data, to ensure direct comparison across backbones. A full description of all the hyperparameters used in COMLN is given in \cref{app:experiment-details}.

\paragraph{Reproducibility of baseline results} To the best of our ability, we have tried to report baseline results from existing work, to limit as much as possible the bias induced by running our own baseline experiments. The references of those works are given in \cref{tab:results-architecture-references}. We still had to run CAVIA and iMAML on the remaining settings, since these results have not been reported in the literature. For both methods, we used the data generation described above.
\begin{itemize}
    \item \emph{CAVIA}: We used the official implementation\footnote{\href{https://github.com/lmzintgraf/cavia/}{\texttt{https://github.com/lmzintgraf/cavia/}}}. We used the hyperparameters reported in \citep{zintgraf2019fast} for \textit{mini}ImageNet, and an architecture with $64$ filters.
    \item \emph{iMAML}: We used the official implementation\footnote{\href{https://github.com/aravindr93/imaml_dev}{\texttt{https://github.com/aravindr93/imaml\_dev}}}. We used the hyperparameters reported in \citep{rajeswaran2019imaml} for \textit{mini}ImageNet 1-shot 5-way.
\end{itemize}

\begin{table}[h]
\centering
\caption{References for the results provided in \cref{tab:results-architecture}: \protect\tikz[baseline=-0.6ex]{\protect\draw[fill=flat-blue!20] circle (0.8ex);} \citep{liu2018tpn}, \protect\tikz[baseline=-0.6ex]{\protect\draw[fill=flat-red!20] circle (0.8ex);} \citep{oh2020boil}, \protect\tikz[baseline=-0.6ex]{\protect\draw[fill=flat-green!20] circle (0.8ex);} \citep{aimen2021taskattended}, \protect\tikz[baseline=-0.6ex]{\protect\draw[fill=flat-purple!20] circle (0.8ex);} \citep{arnold2021uniform}, and \protect\tikz[baseline=-0.6ex]{\protect\draw[fill=flat-yellow!20] circle (0.8ex);} are reported in their respective references (under \emph{Model}). Recall that $\star$ denotes baseline results we executed using the official implementations.}
\label{tab:results-architecture-references}
\begin{adjustbox}{center}
\small
\begin{tabular}{lclcccc}
\toprule
\multirow{2}{*}{\textbf{Model}} & & \multirow{2}{*}{\textbf{Backbone}} & \multicolumn{2}{c}{\textbf{\textit{mini}ImageNet 5-way}} & \multicolumn{2}{c}{\textbf{\textit{tiered}ImageNet 5-way}} \\ \cmidrule(lr){4-5} \cmidrule(lr){6-7}
 & & & \textbf{1-shot} & \textbf{5-shot} & \textbf{1-shot} & \textbf{5-shot} \\ \midrule
MAML \citep{finn17maml} & \cmark & Conv-4 & \cellcolor{flat-yellow!20}$48.70 \pm 1.84^{\phantom{\star}}$ & \cellcolor{flat-yellow!20}$63.11 \pm 0.92^{\phantom{\star}}$ & \cellcolor{flat-blue!20}$51.67 \pm 1.81^{\phantom{\star}}$ & \cellcolor{flat-blue!20}$70.30 \pm 1.75^{\phantom{\star}}$ \\
ANIL \citep{raghu2019anil} & \cmark & Conv-4 & \cellcolor{flat-yellow!20}$46.30 \pm 0.40^{\phantom{\star}}$ & \cellcolor{flat-yellow!20}$61.00 \pm 0.60^{\phantom{\star}}$ & \cellcolor{flat-red!20}$49.35 \pm 0.26^{\phantom{\star}}$ & \cellcolor{flat-red!20}$65.82 \pm 0.12^{\phantom{\star}}$ \\
Meta-SGD \citep{li2017metasgd} & \cmark & Conv-4 & \cellcolor{flat-yellow!20}$50.47 \pm 1.87^{\phantom{\star}}$ & \cellcolor{flat-yellow!20}$64.03 \pm 0.94^{\phantom{\star}}$ & \cellcolor{flat-green!20}$52.80 \pm 0.44^{\phantom{\star}}$ & \cellcolor{flat-green!20}$62.35 \pm 0.26^{\phantom{\star}}$ \\
CAVIA \citep{zintgraf2019fast} & \cmark & Conv-4 & \cellcolor{flat-yellow!20}$51.82 \pm 0.65^{\phantom{\star}}$ & \cellcolor{flat-yellow!20}$65.85 \pm 0.55^{\phantom{\star}}$ & $52.41 \pm 2.64^{\star}$ & $67.55 \pm 2.05^{\star}$ \\
iMAML \citep{rajeswaran2019imaml} & \cmark & Conv-4 & \cellcolor{flat-yellow!20}$49.30 \pm 1.88^{\phantom{\star}}$ & $59.77 \pm 0.73^{\star}$ & $38.54 \pm 1.37^{\star}$ & $60.24 \pm 0.76^{\star}$ \\
MetaOptNet-RR \citep{lee2019metaoptnet} &  & Conv-4 & \cellcolor{flat-yellow!20}$\mathbf{53.23 \pm 0.59^{\phantom{\star}}}$ & \cellcolor{flat-yellow!20}$69.51 \pm 0.48^{\phantom{\star}}$ & \cellcolor{flat-yellow!20}$54.63 \pm 0.67^{\phantom{\star}}$ & \cellcolor{flat-yellow!20}$\mathbf{72.11 \pm 0.59^{\phantom{\star}}}$ \\
MetaOptNet-SVM \citep{lee2019metaoptnet} &  & Conv-4 & \cellcolor{flat-yellow!20}$52.87 \pm 0.57^{\phantom{\star}}$ & \cellcolor{flat-yellow!20}$68.76 \pm 0.48^{\phantom{\star}}$ & \cellcolor{flat-yellow!20}$\mathbf{54.71 \pm 0.67^{\phantom{\star}}}$ & \cellcolor{flat-yellow!20}$71.79 \pm 0.59^{\phantom{\star}}$ \\ \midrule
\textbf{COMLN (Ours)} & \cmark & Conv-4 & $53.01 \pm 0.62^{\phantom{\star}}$ & $\mathbf{70.54 \pm 0.54^{\phantom{\star}}}$ & $54.30 \pm 0.69^{\phantom{\star}}$ & $71.35 \pm 0.57^{\phantom{\star}}$\\ \midrule
MAML \citep{finn17maml} & \cmark & ResNet-12 & \cellcolor{flat-purple!20}$49.92 \pm 0.65^{\phantom{\star}}$ & \cellcolor{flat-purple!20}$63.93 \pm 0.59^{\phantom{\star}}$ & \cellcolor{flat-purple!20}$55.37 \pm 0.74^{\phantom{\star}}$ & \cellcolor{flat-purple!20}$72.93 \pm 0.60^{\phantom{\star}}$\\
ANIL \citep{raghu2019anil} & \cmark & ResNet-12 & \cellcolor{flat-purple!20}$49.65 \pm 0.65^{\phantom{\star}}$ & \cellcolor{flat-purple!20}$59.51 \pm 0.56^{\phantom{\star}}$ & \cellcolor{flat-purple!20}$54.77 \pm 0.76^{\phantom{\star}}$ & \cellcolor{flat-purple!20}$69.28 \pm 0.67^{\phantom{\star}}$\\
MetaOptNet-RR \citep{lee2019metaoptnet} &  & ResNet-12 & \cellcolor{flat-yellow!20}$61.41 \pm 0.61^{\phantom{\star}}$ & \cellcolor{flat-yellow!20}$77.88 \pm 0.46^{\phantom{\star}}$ & \cellcolor{flat-yellow!20}$65.36 \pm 0.71^{\phantom{\star}}$ & \cellcolor{flat-yellow!20}$81.34 \pm 0.52^{\phantom{\star}}$ \\
MetaOptNet-SVM \citep{lee2019metaoptnet} &  & ResNet-12 & \cellcolor{flat-yellow!20}$\mathbf{62.64 \pm 0.61^{\phantom{\star}}}$ & \cellcolor{flat-yellow!20}$\mathbf{78.63 \pm 0.46^{\phantom{\star}}}$ & \cellcolor{flat-yellow!20}$\mathbf{65.99 \pm 0.72^{\phantom{\star}}}$ & \cellcolor{flat-yellow!20}$\mathbf{81.56 \pm 0.53^{\phantom{\star}}}$\\ \midrule
\textbf{COMLN (Ours)} & \cmark & ResNet-12 & $59.26 \pm
0.65^{\phantom{\star}}$ & $77.26 \pm 0.49^{\phantom{\star}}$ & $62.93 \pm 0.71^{\phantom{\star}}$ & $81.13 \pm 0.53^{\phantom{\star}}$ \\ \bottomrule
\end{tabular}
\end{adjustbox}
\end{table}

\newpage
\bibliography{references}

\begin{thebibliography}{78}
\providecommand{\natexlab}[1]{#1}
\providecommand{\url}[1]{\texttt{#1}}
\expandafter\ifx\csname urlstyle\endcsname\relax
  \providecommand{\doi}[1]{doi: #1}\else
  \providecommand{\doi}{doi: \begingroup \urlstyle{rm}\Url}\fi

\bibitem[Aimen et~al.(2021)Aimen, Sidheekh, and
  Krishnan]{aimen2021taskattended}
Aroof Aimen, Sahil Sidheekh, and Narayanan~C Krishnan.
\newblock {Task Attended Meta-Learning for Few-Shot Learning}.
\newblock \emph{arXiv preprint}, 2021.

\bibitem[Arnold \& Sha(2021)Arnold and Sha]{arnold2021embedding}
S{\'e}bastien~MR Arnold and Fei Sha.
\newblock {Embedding Adaptation is Still Needed for Few-Shot Learning}.
\newblock \emph{arXiv preprint}, 2021.

\bibitem[Arnold et~al.(2021)Arnold, Dhillon, Ravichandran, and
  Soatto]{arnold2021uniform}
S{\'e}bastien~MR Arnold, Guneet~S Dhillon, Avinash Ravichandran, and Stefano
  Soatto.
\newblock {Uniform Sampling over Episode Difficulty}.
\newblock \emph{arXiv preprint}, 2021.

\bibitem[Baranchuk(2019)]{baranchuk2019memoryefficientmaml}
Dmitry Baranchuk.
\newblock {Memory Efficient MAML}, 2019.
\newblock URL \url{https://github.com/dbaranchuk/memory-efficient-maml}.

\bibitem[Bengio et~al.(1991)Bengio, Bengio, Cloutier, and Gecsei]{bengio1991}
Yoshua Bengio, Samy Bengio, Jocelyn Cloutier, and Jan Gecsei.
\newblock {Learning a Synaptic Learning Rule}.
\newblock \emph{International Joint Conference on Neural Networks}, 1991.

\bibitem[Bertinetto et~al.(2018)Bertinetto, Henriques, Torr, and
  Vedaldi]{bertinetto2018r2d2}
Luca Bertinetto, Joao~F Henriques, Philip~HS Torr, and Andrea Vedaldi.
\newblock {Meta-learning with Differentiable Closed-Form Solvers}.
\newblock \emph{arXiv preprint}, 2018.

\bibitem[Betts(2010)]{betts2010optimalcontrol}
John~T Betts.
\newblock \emph{{Practical Methods for Optimal Control and Estimation Using
  Nonlinear Programming}}.
\newblock SIAM, 2010.

\bibitem[Biegler(2010)]{Biegler2010}
Lorenz~T. Biegler.
\newblock \emph{Nonlinear Programming}.
\newblock Society for Industrial and Applied Mathematics, January 2010.

\bibitem[B{\"o}hning(1992)]{bohning1992multinomiallogreg}
Dankmar B{\"o}hning.
\newblock {Multinomial Logistic Regression Algorithm}.
\newblock \emph{Annals of the Institute of Statistical Mathematics}, 1992.

\bibitem[Bradbury et~al.(2018)Bradbury, Frostig, Hawkins, Johnson, Leary,
  Maclaurin, and Wanderman-Milne]{jax2018github}
James Bradbury, Roy Frostig, Peter Hawkins, Matthew~James Johnson, Chris Leary,
  Dougal Maclaurin, and Skye Wanderman-Milne.
\newblock {JAX}: composable transformations of {P}ython+{N}um{P}y programs,
  2018.
\newblock URL \url{http://github.com/google/jax}.

\bibitem[Bryson \& Ho(1969)Bryson and Ho]{BrysonHo69}
A.~E. Bryson and Y.~C. Ho.
\newblock \emph{Applied Optimal Control}.
\newblock Blaisdell, New York, 1969.

\bibitem[Butcher(2008)]{butcher2008numerical}
John~Charles Butcher.
\newblock \emph{{Numerical Methods for Ordinary Differential Equations}}.
\newblock Wiley, 2008.

\bibitem[Caracotsios \& Stewart(1985)Caracotsios and
  Stewart]{caracotsios1985sensitivity}
Makis Caracotsios and Warren~E Stewart.
\newblock Sensitivity analysis of initial value problems with mixed odes and
  algebraic equations.
\newblock \emph{Computers \& Chemical Engineering}, 9\penalty0 (4):\penalty0
  359--365, 1985.

\bibitem[Chavent et~al.(1974)Chavent, Goodson, and
  Polis]{chavent1974identification}
G~Chavent, RE~Goodson, and M~Polis.
\newblock Identification of parameter distributed systems.
\newblock \emph{Identification of function parameters in partial differential
  equations}, pp.\  31--48, 1974.

\bibitem[Chen et~al.(2018)Chen, Rubanova, Bettencourt, and
  Duvenaud]{chen2018neuralode}
Ricky T.~Q. Chen, Yulia Rubanova, Jesse Bettencourt, and David Duvenaud.
\newblock {Neural Ordinary Differential Equations}.
\newblock \emph{Advances in Neural Information Processing Systems}, 2018.

\bibitem[Chen et~al.(2020{\natexlab{a}})Chen, Dai, Li, Gao, and
  Song]{chen2020learning2stop}
Xinshi Chen, Hanjun Dai, Yu~Li, Xin Gao, and Le~Song.
\newblock {Learning To Stop While Learning To Predict}.
\newblock In \emph{International Conference on Machine Learning},
  2020{\natexlab{a}}.

\bibitem[Chen et~al.(2020{\natexlab{b}})Chen, Friesen, Behbahani, Doucet,
  Budden, Hoffman, and de~Freitas]{chen2019modular}
Yutian Chen, Abram~L Friesen, Feryal Behbahani, Arnaud Doucet, David Budden,
  Matthew~W Hoffman, and Nando de~Freitas.
\newblock {Modular Meta-Learning with Shrinkage}.
\newblock \emph{Neural Information Processing Systems}, 2020{\natexlab{b}}.

\bibitem[Dormand \& Prince(1980)Dormand and Prince]{dormand1980family}
John~R Dormand and Peter~J Prince.
\newblock {A family of embedded Runge-Kutta formulae}.
\newblock \emph{Journal of computational and applied mathematics}, 1980.

\bibitem[Euler(1913)]{euler1913eulermethod}
Leonhard Euler.
\newblock {De integratione aequationum differentialium per approximationem}.
\newblock \emph{Opera Omnia}, 1913.

\bibitem[Feehery et~al.(1997)Feehery, Tolsma, and Barton]{feehery1997efficient}
William~F Feehery, John~E Tolsma, and Paul~I Barton.
\newblock Efficient sensitivity analysis of large-scale differential-algebraic
  systems.
\newblock \emph{Applied Numerical Mathematics}, 25\penalty0 (1):\penalty0
  41--54, 1997.

\bibitem[Finn(2018)]{finn2018phd}
Chelsea Finn.
\newblock \emph{{Learning to Learn with Gradients}}.
\newblock PhD thesis, UC Berkeley, 2018.

\bibitem[Finn et~al.(2017)Finn, Abbeel, and Levine]{finn17maml}
Chelsea Finn, Pieter Abbeel, and Sergey Levine.
\newblock {Model-Agnostic Meta-Learning for Fast Adaptation of Deep Networks}.
\newblock \emph{International Conference on Machine Learning (ICML)}, 2017.

\bibitem[Flennerhag et~al.(2018)Flennerhag, Moreno, Lawrence, and
  Damianou]{flennerhag2018leap}
Sebastian Flennerhag, Pablo~G Moreno, Neil~D Lawrence, and Andreas Damianou.
\newblock Transferring knowledge across learning processes.
\newblock \emph{arXiv preprint}, 2018.

\bibitem[Flennerhag et~al.(2020)Flennerhag, Rusu, Pascanu, Visin, Yin, and
  Hadsell]{flennerhag2020warpgrad}
Sebastian Flennerhag, Andrei~A Rusu, Razvan Pascanu, Francesco Visin, Hujun
  Yin, and Raia Hadsell.
\newblock {Meta-Learning with Warped Gradient Descent}.
\newblock \emph{International Conference on Learning Representations}, 2020.

\bibitem[Franceschi et~al.(2017)Franceschi, Donini, Frasconi, and
  Pontil]{franceschi2017forward}
Luca Franceschi, Michele Donini, Paolo Frasconi, and Massimiliano Pontil.
\newblock Forward and reverse gradient-based hyperparameter optimization.
\newblock In \emph{International Conference on Machine Learning}, 2017.

\bibitem[Ghiasi et~al.(2018)Ghiasi, Lin, and Le]{ghiasi2018dropblock}
Golnaz Ghiasi, Tsung-Yi Lin, and Quoc~V Le.
\newblock Dropblock: A regularization method for convolutional networks.
\newblock In \emph{Neural Information Processing Systems}, 2018.

\bibitem[Gholami et~al.(2019)Gholami, Keutzer, and
  Biros]{DBLP:journals/corr/abs-1902-10298}
Amir Gholami, Kurt Keutzer, and George Biros.
\newblock {ANODE:} unconditionally accurate memory-efficient gradients for
  neural odes.
\newblock \emph{CoRR}, abs/1902.10298, 2019.
\newblock URL \url{http://arxiv.org/abs/1902.10298}.

\bibitem[Griewank \& Walther(2008)Griewank and Walther]{Griewank2008}
Andreas Griewank and Andrea Walther.
\newblock \emph{Evaluating Derivatives}.
\newblock Society for Industrial and Applied Mathematics, January 2008.

\bibitem[Guo et~al.(2021)Guo, Dietrich, Bertalan, Doncevic, Dahmen, Kevrekidis,
  and Li]{guo2021personalized}
Yue Guo, Felix Dietrich, Tom Bertalan, Danimir~T Doncevic, Manuel Dahmen,
  Ioannis~G Kevrekidis, and Qianxiao Li.
\newblock {Personalized Algorithm Generation: A Case Study in Meta-Learning ODE
  Integrators}.
\newblock \emph{arXiv preprint}, 2021.

\bibitem[He et~al.(2016)He, Zhang, Ren, and Sun]{he2016resnet}
Kaiming He, Xiangyu Zhang, Shaoqing Ren, and Jian Sun.
\newblock {Deep Residual Learning for Image Recognition}.
\newblock In \emph{{Proceedings of the IEEE Conference on Computer Vision and
  Pattern Recognition}}, 2016.

\bibitem[Hennigan et~al.(2020)Hennigan, Cai, Norman, and
  Babuschkin]{haiku2020github}
Tom Hennigan, Trevor Cai, Tamara Norman, and Igor Babuschkin.
\newblock {H}aiku: {S}onnet for {JAX}, 2020.
\newblock URL \url{http://github.com/deepmind/dm-haiku}.

\bibitem[Im et~al.(2019)Im, Jiang, and Verma]{im2019mamlrk}
Daniel~Jiwoong Im, Yibo Jiang, and Nakul Verma.
\newblock {Model-Agnostic Meta-Learning using Runge-Kutta Methods}.
\newblock \emph{arXiv preprint}, 2019.

\bibitem[Javed \& White(2019)Javed and White]{javed2019oml}
Khurram Javed and Martha White.
\newblock {Meta-Learning Representations for Continual Learning}.
\newblock In \emph{Advances in Neural Information Processing Systems}, 2019.

\bibitem[Jiwoong~Im et~al.(2021)Jiwoong~Im, Savin, and Cho]{jiwoong2021online}
Daniel Jiwoong~Im, Cristina Savin, and Kyunghyun Cho.
\newblock {Online hyperparameter optimization by Real-Time Recurrent Learning}.
\newblock \emph{arXiv preprint}, 2021.

\bibitem[Lee et~al.(2019)Lee, Maji, Ravichandran, and
  Soatto]{lee2019metaoptnet}
Kwonjoon Lee, Subhransu Maji, Avinash Ravichandran, and Stefano Soatto.
\newblock {Meta-learning with Differentiable Convex Optimization}.
\newblock In \emph{Proceedings of the IEEE/CVF Conference on Computer Vision
  and Pattern Recognition}, 2019.

\bibitem[Leis \& Kramer(1988)Leis and Kramer]{leis1988simultaneous}
Jorge~R Leis and Mark~A Kramer.
\newblock The simultaneous solution and sensitivity analysis of systems
  described by ordinary differential equations.
\newblock \emph{ACM Transactions on Mathematical Software (TOMS)}, 14\penalty0
  (1):\penalty0 45--60, 1988.

\bibitem[Li et~al.(2018)Li, Yang, Song, and Hospedales]{li2018learning}
Da~Li, Yongxin Yang, Yi-Zhe Song, and Timothy Hospedales.
\newblock {Learning to Generalize: Meta-Learning for Domain Generalization}.
\newblock In \emph{Proceedings of the AAAI Conference on Artificial
  Intelligence}, 2018.

\bibitem[Li et~al.(2021)Li, Wang, Narayan, Kirby, and
  Zhe]{li2021metalearningadjoint}
Shibo Li, Zheng Wang, Akil Narayan, Robert Kirby, and Shandian Zhe.
\newblock {Meta-Learning with Adjoint Methods}.
\newblock \emph{arXiv preprint}, 2021.

\bibitem[Li et~al.(2017)Li, Zhou, Chen, and Li]{li2017metasgd}
Zhenguo Li, Fengwei Zhou, Fei Chen, and Hang Li.
\newblock {Meta-SGD: Learning to learn quickly for few-shot learning}.
\newblock \emph{arXiv preprint}, 2017.

\bibitem[Lions \& Magenes(2012)Lions and Magenes]{lions2012non}
Jacques~Louis Lions and Enrico Magenes.
\newblock \emph{Non-homogeneous boundary value problems and applications: Vol.
  1}, volume 181.
\newblock Springer Science \& Business Media, 2012.

\bibitem[Liu et~al.(2019)Liu, Lee, Park, Kim, Yang, Hwang, and
  Yang]{liu2018tpn}
Yanbin Liu, Juho Lee, Minseop Park, Saehoon Kim, Eunho Yang, Sung~Ju Hwang, and
  Yi~Yang.
\newblock {Learning to Propagate Labels: Transductive Propagation Network for
  Few-Shot Learning}.
\newblock \emph{International Conference on Learning Representations}, 2019.

\bibitem[Maclaurin et~al.(2015)Maclaurin, Duvenaud, and
  Adams]{maclaurin2015gradient}
Dougal Maclaurin, David Duvenaud, and Ryan Adams.
\newblock Gradient-based hyperparameter optimization through reversible
  learning.
\newblock In \emph{International conference on machine learning}, pp.\
  2113--2122. PMLR, 2015.

\bibitem[Maly \& Petzold(1996)Maly and Petzold]{maly1996numerical}
Timothy Maly and Linda~R Petzold.
\newblock Numerical methods and software for sensitivity analysis of
  differential-algebraic systems.
\newblock \emph{Applied Numerical Mathematics}, 20\penalty0 (1-2):\penalty0
  57--79, 1996.

\bibitem[Micaelli \& Storkey(2021)Micaelli and Storkey]{micaelli2021gradient}
Paul Micaelli and Amos Storkey.
\newblock {Gradient-based Hyperparameter Optimization Over Long Horizons}.
\newblock In \emph{{Neural Information Processing Systems}}, 2021.

\bibitem[Nichol et~al.(2018)Nichol, Achiam, and Schulman]{nichol2018reptile}
Alex Nichol, Joshua Achiam, and John Schulman.
\newblock {On First-Order Meta-Learning Algorithms}.
\newblock \emph{arXiv preprint}, 2018.

\bibitem[Oh et~al.(2021)Oh, Yoo, Kim, and Yun]{oh2020boil}
Jaehoon Oh, Hyungjun Yoo, ChangHwan Kim, and Se-Young Yun.
\newblock {BOIL: Towards Representation Change for Few-Shot Learning}.
\newblock \emph{International Conference on Learning Representations}, 2021.

\bibitem[Orvieto \& Lucchi(2019)Orvieto and Lucchi]{orvieto2019shadowing}
Antonio Orvieto and Aurelien Lucchi.
\newblock {Shadowing Properties of Optimization Algorithms}.
\newblock In \emph{Advances in Neural Information Processing Systems}, 2019.

\bibitem[Platt \& Barr(1988)Platt and Barr]{platt88constrained}
John Platt and Alan Barr.
\newblock {Constrained Differential Optimization}.
\newblock In \emph{{Neural Information Processing Systems}}, 1988.

\bibitem[Polyak(1964)]{polyak1964heavyball}
Boris~T Polyak.
\newblock Some methods of speeding up the convergence of iteration methods.
\newblock \emph{{USSR computational mathematics and mathematical physics}},
  1964.

\bibitem[Pontryagin(2018)]{pontryagin2018mathematical}
Lev~Semenovich Pontryagin.
\newblock \emph{Mathematical theory of optimal processes}.
\newblock Routledge, 2018.

\bibitem[Raghu et~al.(2019)Raghu, Raghu, Bengio, and Vinyals]{raghu2019anil}
Aniruddh Raghu, Maithra Raghu, Samy Bengio, and Oriol Vinyals.
\newblock Rapid learning or feature reuse? towards understanding the
  effectiveness of maml.
\newblock \emph{arXiv preprint}, 2019.

\bibitem[Rajeswaran et~al.(2019)Rajeswaran, Finn, Kakade, and
  Levine]{rajeswaran2019imaml}
Aravind Rajeswaran, Chelsea Finn, Sham~M Kakade, and Sergey Levine.
\newblock {Meta-Learning with Implicit Gradients}.
\newblock In \emph{Advances in Neural Information Processing Systems}, 2019.

\bibitem[Ren et~al.(2018)Ren, Triantafillou, Ravi, Snell, Swersky, Tenenbaum,
  Larochelle, and Zemel]{ren2018meta}
Mengye Ren, Eleni Triantafillou, Sachin Ravi, Jake Snell, Kevin Swersky,
  Joshua~B Tenenbaum, Hugo Larochelle, and Richard~S Zemel.
\newblock Meta-learning for semi-supervised few-shot classification.
\newblock In \emph{International Conference on Learning Representations}, 2018.

\bibitem[Rothfuss et~al.(2019)Rothfuss, Lee, Clavera, Asfour, and
  Abbeel]{rothfuss2019promp}
Jonas Rothfuss, Dennis Lee, Ignasi Clavera, Tamim Asfour, and Pieter Abbeel.
\newblock {ProMP: Proximal Meta-Policy Search}.
\newblock \emph{International Conference on Learning Representations}, 2019.

\bibitem[Runge(1895)]{runge1895numerische}
Carl Runge.
\newblock {\"U}ber die numerische aufl{\"o}sung von differentialgleichungen.
\newblock \emph{Mathematische Annalen}, 1895.

\bibitem[Russakovsky et~al.(2015)Russakovsky, Deng, Su, Krause, Satheesh, Ma,
  Huang, Karpathy, Khosla, Bernstein, et~al.]{russakovsky2015imagenet}
Olga Russakovsky, Jia Deng, Hao Su, Jonathan Krause, Sanjeev Satheesh, Sean Ma,
  Zhiheng Huang, Andrej Karpathy, Aditya Khosla, Michael Bernstein, et~al.
\newblock Imagenet large scale visual recognition challenge.
\newblock \emph{International journal of computer vision}, 115\penalty0
  (3):\penalty0 211--252, 2015.

\bibitem[Rusu et~al.(2018)Rusu, Rao, Sygnowski, Vinyals, Pascanu, Osindero, and
  Hadsell]{rusu2018leo}
Andrei~A Rusu, Dushyant Rao, Jakub Sygnowski, Oriol Vinyals, Razvan Pascanu,
  Simon Osindero, and Raia Hadsell.
\newblock {Meta-learning with Latent Embedding Optimization}.
\newblock \emph{arXiv preprint}, 2018.

\bibitem[Santoro et~al.(2016)Santoro, Bartunov, Botvinick, Wierstra, and
  Lillicrap]{santoro2016mann}
Adam Santoro, Sergey Bartunov, Matthew Botvinick, Daan Wierstra, and Timothy
  Lillicrap.
\newblock One-shot learning with memory-augmented neural networks.
\newblock \emph{arXiv preprint}, 2016.

\bibitem[Schmidhuber(1987)]{schmidhuber1987evolutionary}
J{\"u}rgen Schmidhuber.
\newblock \emph{Evolutionary principles in self-referential learning, or on
  learning how to learn: the meta-meta-... hook}.
\newblock PhD thesis, Technische Universit{\"a}t M{\"u}nchen, 1987.

\bibitem[Serban \& Hindmarsh(2003)Serban and Hindmarsh]{serban2003cvodes}
Radu Serban and Alan~C Hindmarsh.
\newblock Cvodes: An ode solver with sensitivity analysis capabilities.
\newblock Technical report, Technical Report UCRL-JP-200039, Lawrence Livermore
  National Laboratory, 2003.

\bibitem[Shaban et~al.(2019)Shaban, Cheng, Hatch, and
  Boots]{shaban2019truncated}
Amirreza Shaban, Ching-An Cheng, Nathan Hatch, and Byron Boots.
\newblock {Truncated Back-propagation for Bilevel Optimization}.
\newblock In \emph{International Conference on Artificial Intelligence and
  Statistics}, 2019.

\bibitem[Snell et~al.(2017)Snell, Swersky, and Zemel]{snell2017protonet}
Jake Snell, Kevin Swersky, and Richard Zemel.
\newblock {Prototypical Networks for Few-shot Learning}.
\newblock In \emph{Advances in Neural Information Processing Systems}, 2017.

\bibitem[Su et~al.(2014)Su, Boyd, and Candes]{su2014differential}
Weijie Su, Stephen Boyd, and Emmanuel Candes.
\newblock {A Differential Equation for Modeling Nesterov’s Accelerated
  Gradient Method: Theory and Insights}.
\newblock \emph{Advances in Neural Information Processing Systems}, 2014.

\bibitem[Sung et~al.(2018)Sung, Yang, Zhang, Xiang, Torr, and
  Hospedales]{sung2018learning}
Flood Sung, Yongxin Yang, Li~Zhang, Tao Xiang, Philip~HS Torr, and Timothy~M
  Hospedales.
\newblock Learning to compare: Relation network for few-shot learning.
\newblock In \emph{Proceedings of the IEEE conference on computer vision and
  pattern recognition}, 2018.

\bibitem[Sutton(1992)]{Sutton1992}
Richard~S. Sutton.
\newblock {Adapting Bias by Gradient Descent: An Incremental Version of
  Delta-Bar-Delta}.
\newblock In \emph{Proceedings of the 10th National Conference on Artificial
  Intelligence}, 1992.

\bibitem[Thrun \& Pratt(2012)Thrun and Pratt]{thrun2012learningtolearn}
Sebastian Thrun and Lorien Pratt.
\newblock \emph{Learning to learn}.
\newblock Springer Science \& Business Media, 2012.

\bibitem[Tian et~al.(2020)Tian, Wang, Krishnan, Tenenbaum, and
  Isola]{tian2020rethinking}
Yonglong Tian, Yue Wang, Dilip Krishnan, Joshua~B Tenenbaum, and Phillip Isola.
\newblock {Rethinking Few-Shot Image Classification: a Good Embedding Is All
  You Need?}
\newblock 2020.

\bibitem[Vinyals et~al.(2016)Vinyals, Blundell, Lillicrap, Wierstra,
  et~al.]{vinyals2016matching}
Oriol Vinyals, Charles Blundell, Timothy Lillicrap, Daan Wierstra, et~al.
\newblock Matching networks for one shot learning.
\newblock \emph{Neural Information Processing Systems}, 29:\penalty0
  3630--3638, 2016.

\bibitem[Von~Oswald et~al.(2021)Von~Oswald, Zhao, Kobayashi, Schug, Caccia,
  Zucchet, and Sacramento]{von2021sparsemaml}
Johannes Von~Oswald, Dominic Zhao, Seijin Kobayashi, Simon Schug, Massimo
  Caccia, Nicolas Zucchet, and Jo{\~a}o Sacramento.
\newblock {Learning where to learn: Gradient sparsity in meta and continual
  learning}.
\newblock \emph{Advances in Neural Information Processing Systems}, 2021.

\bibitem[Wiggins(2003)]{wiggins1990dynamicalsystems}
Stephen Wiggins.
\newblock \emph{{Introduction to Applied Nonlinear Dynamical Systems and
  Chaos}}, volume~2.
\newblock Springer, 2003.

\bibitem[Williams \& Zipser(1989)Williams and Zipser]{Williams1989}
Ronald~J. Williams and David Zipser.
\newblock {A Learning Algorithm for Continually Running Fully Recurrent Neural
  Networks}.
\newblock \emph{Neural Computation}, 1989.

\bibitem[Wilson et~al.(2016)Wilson, Recht, and Jordan]{wilson2016lyapunov}
Ashia~C Wilson, Benjamin Recht, and Michael~I Jordan.
\newblock {A Lyapunov Analysis of Momentum Methods in Optimization}.
\newblock \emph{arXiv preprint}, 2016.

\bibitem[Xu et~al.(2021)Xu, Chen, and Karbasi]{xu2021metalearningcontinuous}
Ruitu Xu, Lin Chen, and Amin Karbasi.
\newblock { Meta Learning in the Continuous Time Limit }.
\newblock In \emph{Proceedings of The 24th International Conference on
  Artificial Intelligence and Statistics}. PMLR, 2021.

\bibitem[Xu et~al.(2018)Xu, van Hasselt, and Silver]{Xu2018}
Zhongwen Xu, Hado van Hasselt, and David Silver.
\newblock {Meta-Gradient Reinforcement Learning}.
\newblock In \emph{Advances in Neural Information Processing Systems}, 2018.

\bibitem[Yu et~al.(2020)Yu, Kumar, Gupta, Levine, Hausman, and
  Finn]{yu2020gradientsrugery}
Tianhe Yu, Saurabh Kumar, Abhishek Gupta, Sergey Levine, Karol Hausman, and
  Chelsea Finn.
\newblock {Gradient Surgery for Multi-Task Learning}.
\newblock \emph{Neural Information Processing Systems}, 2020.

\bibitem[Zhang et~al.(2021)Zhang, Li, Ye, Feng, and Ye]{zhang2021metanode}
Baoquan Zhang, Xutao Li, Yunming Ye, Shanshan Feng, and Rui Ye.
\newblock {MetaNODE: Prototype Optimization as a Neural ODE for Few-Shot
  Learning}.
\newblock \emph{arXiv preprint}, 2021.

\bibitem[Zhou et~al.(2021)Zhou, Wang, Xian, Chen, and Xu]{zhou2021metantk}
Yufan Zhou, Zhenyi Wang, Jiayi Xian, Changyou Chen, and Jinhui Xu.
\newblock {Meta-Learning with Neural Tangent Kernels}.
\newblock \emph{International Conference on Learning Representations}, 2021.

\bibitem[Zintgraf et~al.(2019)Zintgraf, Shiarli, Kurin, Hofmann, and
  Whiteson]{zintgraf2019fast}
Luisa Zintgraf, Kyriacos Shiarli, Vitaly Kurin, Katja Hofmann, and Shimon
  Whiteson.
\newblock Fast context adaptation via meta-learning.
\newblock In \emph{International Conference on Machine Learning}, pp.\
  7693--7702. PMLR, 2019.

\end{thebibliography}
\bibliographystyle{iclr2022_conference}

\newpage
\appendix

{\centering
{\LARGE\bf Appendix\par}}

\DeclareRobustCommand{\protectbm}[1]{\bm{#1}}
\pdfstringdefDisableCommands{%
  \renewcommand{\protectbm}[1]{#1}%
}
\vspace{2em}
The Appendix is organized as follows: in \cref{app:algorithmic-details} we give the pseudo-code for meta-training and meta-testing COMLN, along with minimal code in JAX \citep{jax2018github}. In \cref{app:proofs-meta-gradients} we prove the decomposition of the Jacobians introduced in \cref{sec:memory-efficient-meta-gradients}, and we give the exact dynamics of $\vs_{m}(t)$ and $\vz_{t}[i, j, m]$ in the decomposition of $d\mW(t)/d\vphi_{m}$, omitted from the main text for concision. In \cref{app:projection} we show how the total derivatives may be computed from the Jacobian matrices without ever having to form them explicitly, hence maintaining the memory-efficiency described in \cref{sec:memory-efficiency}. In \cref{app:proofs-stability}, we show that unlike the adjoint method described in \cref{sec:local-sensitivity-analysis-odes}, the algorithm used in COMLN to compute the meta-gradients based on the forward sensitivity method is stable and guaranteed not to diverge. Finally in \cref{app:experiment-details} we give additional details regarding the experiments reported in \cref{sec:experiments}, together with additional analyses and results on the dataset introduced in \citet{rusu2018leo}.

\section{Algorithmic details}
\label{app:algorithmic-details}

\subsection{Meta-training pseudo-code}
\label{app:meta-training}
We give in \cref{alg:comln-meta-training} the pseudo-code for meta-training COMLN, based on a distribution of tasks $p(\tau)$, with references to the relevant propositions developed in \cref{app:proofs-meta-gradients,app:projection}. Note that to simplify the presentation in \cref{alg:comln-meta-training}, we introduce the notations $\nabla_{\tau}$ and $\partial\gL_{\tau}$ to denote the total derivative of the outer-loss, and the partial derivative of the loss $\gL$ respectively, both computed at the adapted parameters $\mW_{\tau}(T)$. For example:
\begin{align*}
    \nabla_{\tau}\mW_{0} &= \frac{d\gL\big(\mW_{\tau}(T);f_{\Phi}(\Dtest)\big)}{d\mW_{0}} & \partial_{\mW(T)}\gL_{\tau}^{\mathrm{train}} &= \frac{\partial \gL\big(\mW_{\tau}(T);f_{\Phi}(\Dtrain)\big)}{\partial \mW(T)}
\end{align*}

\begin{algorithm}[H]
\caption{COMLN -- Meta-training}
\label{alg:comln-meta-training}
\begin{algorithmic}
\Require A task distribution $p(\tau)$
\State Initialize randomly $\Phi$ and $\mW_{0}$. Initialize $T$ to a small value $\varepsilon > 0$.
\Loop
\State Sample a batch of tasks $\gB \sim p(\tau)$
\ForAll{$\tau \in \gB$}
\State Embed the training set: $f_{\Phi}(\Dtrain) = \{(\vphi_{m}, \vy_{m})\}_{m=1}^{M}$
\State $\vs(T), \mB_{T}, \vz_{T} \leftarrow \textrm{ODESolve}([\vzero, \vzero, \vzero], \textrm{\textsc{Dynamics}}, 0, T)$ \Comment{\cref{alg:comln-dynamics-meta-training}}
\State $\mW_{\tau}(T) \leftarrow \mW_{0} - \sum_{m}\vs_{m}(T)\vphi_{m}^{\top}$ \Comment{\cref{prop:decomposition-W}}
\State
\State Embed the test set: $f_{\Phi}(\Dtest)$
\State Compute the partial derivatives: $\partial_{\mW(T)}\gL_{\tau}^{\mathrm{train}}$ \& $\partial_{\mW(T)}\gL_{\tau}^{\mathrm{test}}$
\State $\nabla_{\tau}\mW_{0} \leftarrow \textrm{\textsc{Project}}_{\mW_{0}}\big(\partial_{\mW(T)}\gL_{\tau}^{\mathrm{test}}, \mB_{T}, \vphi\big)$\Comment{\cref{prop:projection-W0}}
\State $\nabla_{\tau}\vphi_{m} \leftarrow \textrm{\textsc{Project}}_{\vphi_{m}}\big(\partial_{\mW(T)}\gL_{\tau}^{\mathrm{test}}, \vs(T), \mB_{T}, \vz_{T}, \vphi\big) + \partial_{\vphi_{m}}\gL_{\tau}^{\mathrm{test}}$ \Comment{\cref{prop:projection-phi}}
\State $\vphantom{\big(}\nabla_{\tau}\Phi \leftarrow$ Backpropagation through $f_{\Phi}$, starting with $\nabla_{\tau}\vphi_{m}$ for all $m$
\State $\nabla_{\tau}T \leftarrow \big(\partial_{\mW(T)}\gL_{\tau}^{\mathrm{test}}\big)^{\top}\big(\partial_{\mW(T)}\gL_{\tau}^{\mathrm{train}}\big)$\Comment{\cref{prop:proof-meta-gradient-T}}
\EndFor
\State Update the meta-parameters:\vphantom{$\frac{\alpha}{|\gB|}$}
\State $\qquad \mW_{0} \leftarrow \mW_{0} - \frac{\alpha}{|\gB|} \sum_{\tau \in \gB}\nabla_{\tau}\mW_{0}$
\State $\qquad \Phi \leftarrow \Phi - \frac{\alpha}{|\gB|} \sum_{\tau\in\gB}\nabla_{\tau}\Phi$
\State $\qquad T \leftarrow T - \frac{\alpha}{|\gB|} \sum_{\tau \in \gB}\nabla_{\tau}T$
\EndLoop
\end{algorithmic}
\end{algorithm}

The dynamical system followed during adaptation is given in \cref{prop:decomposition-W,prop:decomposition-W0,prop:decomposition-phi} and is summarized in \cref{alg:comln-dynamics-meta-training}. The solution of this dynamical system is used not only to find the adapted parameters $\mW_{\tau}(T)$, but also to get the quantities necessary to compute the meta-gradients for the updates of the meta-parameters $\mW_{0}$ \& $\Phi$. The integration of this dynamical system, named \textsc{ODESolve} in \cref{alg:comln-meta-training}, is done in practice using a numerical solver such as Runge-Kutta methods.
\begin{algorithm}[H]
\caption{Dynamical system}
\label{alg:comln-dynamics-meta-training}
\begin{algorithmic}
\Function{Dynamics}{$\mW_{0}, f_{\Phi}(\Dtrain), \vs(t)$}
\State $\mW_{\tau}(t) \leftarrow \mW_{0} - \sum_{m}\vs_{m}(t)\vphi_{m}^{\top}$ \Comment{\cref{prop:decomposition-W}}
\State $\vp_{m}(t) \leftarrow \softmax(\mW_{\tau}(t)\vphi_{m})$
\State $\mA_{m}(t) \leftarrow \big(\diag(\vp_{m}(t)) - \vp_{m}(t)\vp_{m}(t)^{\top}\big) / M$
\State
\State $\dfrac{d\vs_{m}}{dt} \leftarrow \dfrac{1}{M}\big(\vp_{m}(t) - \vy_{m}\big)$ \Comment{\cref{prop:decomposition-W}}
\State $\displaystyle\dfrac{d\mB_{t}[i, j]}{dt} \leftarrow \mathbbm{1}(i=j)\mA_{i}(t) - \mA_{i}(t)\sum_{m=1}^{M}\big(\vphi_{i}^{\top}\vphi_{m}\big)\mB_{t}[m, j]$ \Comment{\cref{prop:decomposition-W0,prop:decomposition-phi} ($\hookleftarrow$)}
\State $\displaystyle \dfrac{d\vz_{t}[i,j,m]}{dt} \leftarrow -\mA_{i}(t)\bigg[\mathbbm{1}(i=j)\vs_{m}(t) + \mathbbm{1}(i=m)\vs_{j}(t) + \sum_{k=1}^{M}\big(\vphi_{i}^{\top}\vphi_{k}\big)\vz_{t}[k,j,m]\bigg]$
\State \textbf{return} $d\vs/dt, d\mB_{t}/dt, d\vz_{t}/dt$
\EndFunction
\end{algorithmic}
\end{algorithm}

In \cref{alg:comln-projection-W0,alg:comln-projection-phi}, we give the procedures responsible for the projection of the partial derivatives $\partial_{\mW(T)}\gL_{\tau}^{\mathrm{test}}$ onto the Jacobian matrices $d\mW(T)/d\mW_{0}$ and $d\mW(T)/d\vphi_{m}$ respectively, based on \cref{prop:projection-W0,prop:projection-phi}. Note that \cref{alg:comln-projection-phi} also depends on the initial conditions $\mW_{0}$, which is implicitly assumed here for clarity of presentation. It is interesting to see that both functions use the same matrix $\mC$, and therefore this can be computed only once and reused for both projections.

\begin{figure}[h]
\begin{minipage}{0.45\linewidth}
\begin{algorithm}[H]
\caption{Projection onto $d\mW(T)/d\mW_{0}$}
\label{alg:comln-projection-W0}
\begin{algorithmic}
\Function{Project${}_{\mW_{0}}$}{$\mV(T), \mB_{T}, \vphi$}
\State $\displaystyle \mC_{j} \leftarrow \sum_{i=1}^{M}\vphi_{i}^{\top}\mV(T)^{\top}\mB_{T}[i, j]$
\State \textbf{return} $\mV(T) - \mC^{\top}\vphi$
\EndFunction
\State \vspace{2.2em}
\end{algorithmic}
\end{algorithm}
\end{minipage}\hfill
\begin{minipage}{0.53\linewidth}
\begin{algorithm}[H]
\caption{Projection onto $d\mW(T)/d\vphi_{m}$}
\label{alg:comln-projection-phi}
\begin{algorithmic}
\Function{Project${}_{\vphi_{m}}$}{$\mV(T), \vs(T), \mB_{T}, \vz_{T}, \vphi$}
\State $\displaystyle \mC_{j} \leftarrow \sum_{i=1}^{M}\vphi_{i}^{\top}\mV(T)^{\top}\mB_{T}[i, j]$
\State $\displaystyle \mD_{m,j} \leftarrow \sum_{i=1}^{M}\vz_{T}[i,j,m]^{\top}\mV(T)\vphi_{i}$
\State \textbf{return} $-\big[\vs_{m}(T)^{\top}\mV(T) + \mC_{m}\mW_{0} + \mD_{m}\vphi\big]$
\EndFunction
\end{algorithmic}
\end{algorithm}
\end{minipage}
\end{figure}

Finally in \cref{alg:comln-meta-test,alg:comln-dynamics-meta-test}, we show how to use COMLN at meta-test time on a novel task, based on the learned meta-parameters $\mW_{0}$, $\Phi$ and $T$. This simply corresponds to integrating a dynamical system in $\mW(t)$ (equivalently, in $\vs_{m}(t)$, see \cref{prop:decomposition-W}). Note that for adaptation during meta-testing, there is no need to compute either $\mB_{t}[i, j]$ or $\vz_{t}[i, j, m]$, since these are only necessary to compute the meta-gradients during meta-training.

\begin{figure}[h]
\begin{minipage}{0.5\linewidth}
\begin{algorithm}[H]
\caption{COMLN -- Meta-test}
\label{alg:comln-meta-test}
\begin{algorithmic}
\Require A task $\tau$ with a dataset $\Dtrain$
\Require Meta-parameters $\Phi$, $\mW_{0}$ \& $T$
\State Embed the training dataset: $f_{\Phi}(\Dtrain)$
\State $\vs(T) \leftarrow \mathrm{ODESolve}(\vzero, \mathrm{\textsc{Dynamics}}, 0, T)$
\State $\mW_{\tau}(T) \leftarrow \mW_{0} - \sum_{m}\vs_{m}(T)\vphi_{m}^{\top}$
\State \textbf{return} $\mW_{\tau}(T) \circ f_{\Phi}$
\end{algorithmic}
\end{algorithm}
\end{minipage}\hfill
\begin{minipage}{0.48\linewidth}
\begin{algorithm}[H]
\caption{Dynamical system for adaptation}
\label{alg:comln-dynamics-meta-test}
\begin{algorithmic}
\Function{Dynamics}{$\mW_{0}, f_{\Phi}(\Dtrain), \vs(t)$}
\State $\mW_{\tau}(t) \leftarrow \mW_{0} - \sum_{m}\vs_{m}(t)\vphi_{m}^{\top}$
\State $\vp_{m}(t) \leftarrow \softmax(\mW_{\tau}(t)\vphi_{m})$
\State $d\vs_{m}/dt \leftarrow \big(\vp_{m}(t) - \vy_{m}\big) / M$
\State \textbf{return} $d\vs/dt$
\EndFunction
\end{algorithmic}
\end{algorithm}
\end{minipage}
\end{figure}

\newpage
\subsection{Source code}
\label{app:source-code}
We provide a snippet of code written in JAX \citep{jax2018github} in order to compute the adapted parameters, based on \cref{alg:comln-meta-training,alg:comln-dynamics-meta-training}. The \texttt{dynamics} function not only computes the vectors $\vs_{m}(t)$ necessary to compute $\mW(t)$ (see \cref{app:proof-W}), but also $\mB_{t}[i,j]$ and $\vz_{t}[i,j,m]$ jointly in order to compute the meta-gradients. This snippet shows that various necessary quantities can be precomputed ahead of adaptation, such as the Gram matrix \texttt{gram}.

\begin{code}
\begin{minted}{python}
import jax.numpy as jnp

from jax import vmap, grad, nn, ops, tree_util
from jax.experimental import ode
from collections import namedtuple

State = namedtuple('State', ['s', 'B', 'z'])

M, N = train_inputs.shape[0], train_labels.shape[1]
gram = jnp.matmul(train_inputs, train_inputs.T)
logits_0 = jnp.matmul(train_inputs, W_0.T)
diag = jnp.diag_indices(M)

def dynamics(state, _):
    preds = nn.softmax(logits_0 - jnp.matmul(gram, state.s), axis=1)
    A = (vmap(jnp.diag)(preds) - vmap(jnp.outer)(preds, preds)) / M

    # Update of s
    ds = (preds - train_labels) / M

    # Update of B
    cross_prod = jnp.einsum('ikn,im,mjnl->ijkl', A, gram, state.B)
    dB = ops.index_add(-cross_prod, diag, A)

    # Update of z
    cross_prod = jnp.einsum('iln,ik,jmn->ijml', A, gram, state.z)
    A_s = jnp.einsum('ikl,jl->ijk', A, state.s)
    dz = ops.index_add(cross_prod, diag, A_s)
    dz = ops.index_add(dz, (diag[0], None, diag[1]), A_s)

    return State(s=ds, B=dB, z=-dz)

state_0 = State(
    s=jnp.zeros((M, N)),
    B=jnp.zeros((M, M, N, N)),
    z=jnp.zeros((M, M, M, N))
)
solution = ode.odeint(dynamics, state_0, jnp.array([0., T]))
state_T = tree_util.tree_map(lambda x: x[-1], solution)
W_T = W_0 - jnp.matmul(state_T.s.T, train_inputs)
\end{minted}
  \caption{Snippet of code in JAX to compute the adapted parameter \texttt{W\_T} for a given task specified by the embedded training set \texttt{train\_inputs} \& \texttt{train\_labels} (note that here \texttt{train\_inputs} are the embedding vectors returned by $f_{\Phi}$) and an initialization \texttt{W\_0}.}
  \label{code:jax-snippet-gradient-flow}
\end{code}

We want to emphasize that all the information required to compute the meta-gradients are available in \texttt{state\_T} found after integration of the \texttt{dynamics} function forward in time, and does not require any additional backward pass (apart from backpropagation through the embedding network $f_{\Phi}$). In particular, the meta-gradients wrt. the initial conditions $\mW_{0}$ and the wrt. the time horizon $T$ can be computed using the snippet of code available in Code Snippet \ref{code:jax-snippet-meta-gradients}.

\begin{code}
\begin{minted}{python}
grads_test = grad(cross_entropy_loss)(W_T, test_inputs, test_labels)
grads_train = grad(cross_entropy_loss)(W_T, train_inputs, train_labels)

C = jnp.einsum('in,kn,ijkl->jl', train_inputs, grads_test, state_T.B)

grads_W_0 = grads_test - jnp.matmul(C.T, train_inputs)
grads_T = -jnp.vdot(grads_train, grads_test)
\end{minted}
  \caption{Snippet of code in JAX to compute the meta-gradients wrt. the meta-parameters \texttt{W\_0} and \texttt{T}, based on \texttt{state\_T} found above. See \cref{app:proof-T} \& \cref{app:projection} for details.}
  \label{code:jax-snippet-meta-gradients}
\end{code}

The code to compute the meta-gradients wrt. the meta-parameters of the embedding network $\Phi$ is not included here for clarity, as it involves non-minimal dependencies on Haiku \citep{haiku2020github} in order to backpropagate the error through the backbone network. The full code is available at \href{https://github.com/tristandeleu/jax-comln}{\texttt{https://github.com/tristandeleu/jax-comln}}.

\section{Proofs of memory efficient meta-gradients}
\label{app:proofs-meta-gradients}
In this section, we will prove the results presented in \cref{sec:memory-efficient-meta-gradients} in a slightly more general case, where the loss function $\gL$ is the cross-entropy loss regularized with a proximal term around the initialization \citep{rajeswaran2019imaml}:
\begin{align}
    \gL\big(\mW;f_{\Phi}(\Dtrain)\big) &= -\frac{1}{M}\sum_{m=1}^{M}\vy_{m}^{\top}\log \vp_{m} + \frac{\lambda}{2}\|\mW - \mW_{0}\|_{F}^{2},
    \label{eq:regularized-cross-entropy}
\end{align}
where $\vp_{m} = \softmax(\mW\vphi_{m})$, $\|\cdot\|_{F}$ is the Frobenius norm, and $\lambda$ is a regularization constant. Note that we can recover the setting presented in the main paper by setting $\lambda = 0$. The core idea of computing the meta-gradient efficiently is based on the decomposition of the gradient (as well as the Hessian matrix) of the cross-entropy loss \citep{bohning1992multinomiallogreg}. We recall this decomposition as the following lemma, extended to include the regularization term:
\begin{lemma}[\citealp{bohning1992multinomiallogreg}]
Let $f_{\Phi}(\Dtrain) = \{(\vphi_{m}, \vy_{m})\}_{m=1}^{M}$ be the embedded training set through the embedding network $f_{\Phi}$, and $\gL$ the regularized cross-entropy loss defined in \cref{eq:regularized-cross-entropy}. The gradient $\nabla\gL$ and the Hessian matrix $\nabla^{2}\gL$ of the regularized cross-entropy loss can be written as
\begin{align*}
    \nabla\gL\big(\mW;f_{\Phi}(\Dtrain)\big) &= \frac{1}{M}\sum_{m=1}^{M}\big(\vp_{m} - \vy_{m}\big)\vphi_{m}^{\top} + \lambda (\mW - \mW_{0})\\
    \nabla^{2}\gL\big(\mW;f_{\Phi}(\Dtrain)\big) &= \sum_{m=1}^{M}\mA_{m} \otimes \vphi_{m}\vphi_{m}^{\top} + \lambda \mI,
\end{align*}
where $\mA_{m} = \big(\diag(\vp_{m}) - \vp_{m}\vp_{m}^{\top}\big) / M$ are $N \times N$ matrices, and $\otimes$ is the Kronecker product.
\label{lem:decomposition-gradient-hessian}
\end{lemma}
This lemma is particularly useful in the context of few-shot learning since it reduces the characterization of the gradient and the Hessian from quantities of size $N\times d$ and $Nd \times Nd$ respectively (where $d$ is the dimension of the embedding vectors $\vphi$) to $M$ objects whose size is independent of $d$---the embedding vectors $\vphi$ being independent of $\mW$. Typically in few-shot learning, $M \ll d$.

In order to avoid higher-order tensors when defining the different Jacobians and Hessians, we always consider them as matrices, with possibly an implicit ``flattening'' operation. For example here even though $\mW$ is a $N \times d$ matrix, we treat the Hessian $\nabla^{2}\gL$ as a $Nd \times Nd$ matrix, as opposed to a 4D tensor. When the context is required, we will make this transformation from a higher-order tensor to a matrix explicit with an encoding of indices. Moreover throughout this section, we will only consider the computation of the meta-gradients for a single task $\tau$, and therefore we will often drop the explicit dependence of the different objects on $\tau$ (e.g. we will write $\mW(t)$ instead of $\mW_{\tau}(t)$); the meta-gradients are eventually averaged over a batch of tasks for the update of the outer-loop, see \cref{app:algorithmic-details} for details in the pseudo-code. Finally, since we only consider adaptation of the last layer of the neural network, the presentation here is always made in the context of an embedded training set $f_{\Phi}(\Dtrain) = \{(\vphi_{m}, \vy_{m})\}_{m=1}^{M}$ that went through the backbone $f_{\Phi}$.

\subsection{Decomposition of $\protectbm{W}(t)$ for parameter adaptation}
\label{app:proof-W}
As a direct application of \cref{lem:decomposition-gradient-hessian}, we first decompose the parameters $\mW(t)$ into smaller quantities $\vs_{m}(t)$ that follow the dynamics defined in Prop.~\ref{prop:decomposition-W}. Although this decomposition is equivalent to $\mW(t)$, in practice solving a smaller dynamical system in $\vs_{m}(t)$ improves the efficiency of our method.
\begin{proposition}
    Let $f_{\Phi}(\Dtrain) = \{(\vphi_{m}, \vy_{m})\}_{m=1}^{M}$ be the embedded training set through the embedding network $f_{\Phi}$, and $\mW(t)$ be the solution of the following dynamical system
    \begin{align*}
        \frac{d\mW}{dt} &= -\nabla \gL\big(\mW(t);f_{\Phi}(\Dtrain)\big) & \mW(0) &= \mW_{0},
    \end{align*}
    where the loss function is the regularized cross-entropy loss defined in \cref{eq:regularized-cross-entropy}. The solution $\mW(t)$ of this dynamical system can be written as
    \begin{equation*}
        \mW(t) = \mW_{0} - \sum_{m=1}^{M}\vs_{m}(t)\vphi_{m}^{\top},
    \end{equation*}
    where for all $m$, $\vs_{m}(t)$ is the solution of the following dynamical system:
    \begin{align*}
        \frac{d\vs_{m}}{dt} &= \frac{1}{M}\big(\vp_{m}(t) - \vy_{m}\big) - \lambda \vs_{m}(t) & \vs_{m}(0) &= \vzero,
    \end{align*}
    and $\vp_{m}(t) = \softmax(\mW(t)\vphi_{m})$ is the vector of predictions returned by the network at time $t$.
    \label{prop:decomposition-W}
\end{proposition}

\begin{proof}
    Using \cref{lem:decomposition-gradient-hessian}, the function $\mW(t)$ is the solution of the following differential equation:
    \begin{align*}
        \frac{d\mW}{dt} &= -\nabla \gL\big(\mW(t);f_{\Phi}(\Dtrain)\big)\\
        &= -\frac{1}{M}\sum_{m=1}^{M}\big(\vp_{m}(t) - \vy_{m}\big)\vphi_{m}^{\top} - \lambda \big(\mW(t) - \mW_{0}\big) & \mW(0) &= \mW_{0}.
    \end{align*}
    The proof relies on the unicity of the solution of a given autonomous differential equation given a particular choice of initial conditions \citep[][Prop. 7.4.2]{wiggins1990dynamicalsystems}. In other words, if we can find another function $\widetilde{\mW}(t)$ that also satisfies the differential equation above with the initial conditions $\widetilde{\mW}(0) = \mW_{0}$, then it means that for all $t$ we have $\widetilde{\mW}(t) = \mW(t)$. Suppose that this function $\widetilde{\mW}(t)$ can be written as
    \begin{equation*}
        \widetilde{\mW}(t) = \mW_{0} - \sum_{m=1}^{M}\vs_{m}(t)\vphi_{m}^{\top},
    \end{equation*}
    where $\vs_{m}$ satisfies the dynamical system defined in the statement of \cref{prop:decomposition-W}. Then we have
    \begin{align*}
        -\nabla \gL\big(\widetilde{\mW}(t);f_{\Phi}(\Dtrain)\big) &= -\frac{1}{M}\sum_{m=1}^{M}\big(\vp_{m}(t) - \vy_{m}\big)\vphi_{m}^{\top} - \lambda\big(\widetilde{\mW}(t) - \mW_{0}\big)\\
        &= -\sum_{m=1}^{M}\frac{d\vs_{m}}{dt}\vphi_{m}^{\top} - \underbrace{\lambda \sum_{m=1}^{M}\vs_{m}(t)\vphi_{m}^{\top} - \lambda\big(\widetilde{\mW}(t) - \mW_{0}\big)}_{=\,\vzero}\\
        &= -\sum_{m=1}^{M}\frac{d\vs_{m}}{dt}\vphi_{m}^{\top} = \frac{d\widetilde{\mW}}{dt}
    \end{align*}
    We have shown that $\widetilde{\mW}(t)$ follows the same dynamics as $\mW(t)$. Moreover, using the initial conditions $\vs_{m}(0) = \vzero$, it is clear that $\widetilde{\mW}(0) = \mW_{0}$, which are the same initial conditions as the equation satisfied by $\mW(t)$. Therefore we have $\widetilde{\mW}(t) = \mW(t)$ for all $t$, showing the expected decomposition of $\mW(t)$.
\end{proof}

\subsection{Decomposition of the Jacobian matrix $d\protectbm{W}(t)/d\protectbm{W}_{0}$}
\label{app:proof-W0}
The core objective of gradient-based meta-learning methods is the capacity to compute the meta-gradients wrt. the initial conditions $\mW_{0}$. In order to compute this meta-gradient using forward-mode automatic differentiation, we want to first compute the Jacobian matrix $d\mW(t)/d\mW_{0}$. We show that this Jacobian can be decomposed into smaller quantities that follow the dynamics in \cref{prop:decomposition-W0}.
\begin{proposition}
    Let $f_{\Phi}(\Dtrain) = \{(\vphi_{m}, \vy_{m})\}_{m=1}^{M}$ be the embedded training set through the embedding network $f_{\Phi}$. The Jacobian matrix $d\mW(t)/d\mW_{0}$ can be written as
    \begin{equation*}
        \frac{d\mW(t)}{d\mW_{0}} = \mI - \sum_{i=1}^{M}\sum_{j=1}^{M}\mB_{t}[i, j]\otimes \vphi_{i}\vphi_{j}^{\top},
    \end{equation*}
    where $\otimes$ is the Kronecker product, and for all $i, j$, $\mB_{t}[i,j]$ is a $N \times N$ matrix which the solution of the following dynamical system
    \begin{align*}
        \frac{d\mB_{t}[i, j]}{dt} &= \mathbbm{1}(i = j)\mA_{i}(t) - \lambda \mB_{t}[i, j] - \mA_{i}(t)\sum_{m=1}^{M}\big(\vphi_{i}^{\top}\vphi_{m}\big)\mB_{t}[m, j] & \mB_{0}[i, j] &= \vzero,
    \end{align*}
    and $\mA_{i}(t) = \big(\diag(\vp_{i}(t)) - \vp_{i}(t)\vp_{i}(t)^{\top}\big)/M$ are defined using the vectors of predictions at time $t$ $\vp_{i}(t) = \softmax\big(\mW(t)\vphi_{i}\big)$.
    \label{prop:decomposition-W0}
\end{proposition}

\begin{proof}
    We will use the forward sensitivity equation from \cref{sec:local-sensitivity-analysis-odes} in order to derive this new dynamical system over $\mB_{t}[i, j]$. Recall that to simplify the notations we can write the gradient vector field followed during adaptation as
    \begin{equation*}
        \frac{d\mW}{dt} = g\big(\mW(t);f_{\Phi}(\Dtrain)\big) \triangleq -\nabla\gL\big(\mW(t);f_{\Phi}(\Dtrain)\big).
    \end{equation*}
    Introducing the matrix-valued function $\gS(t) = d\mW(t)/d\mW_{0}$ as a sensitivity state, we can use the forward sensitivity equations and see that $\gS(t)$ satisfies the following equation
    \begin{align*}
        \frac{d\gS}{dt} &= \frac{\partial g\big(\mW(t)\big)}{\partial \mW(t)}\gS(t) + \frac{\partial g\big(\mW(t)\big)}{\partial \mW_{0}} & \gS(0) &= \mI\\
        &= -\nabla^{2}\gL\big(\mW(t);f_{\Phi}(\Dtrain)\big)\gS(t) + \lambda \mI.
    \end{align*}
    The rest of the proof is based on the unicity of the solution of a given autonomous\footnote{While the dynamical system defined here for the sensitivity state $\gS(t)$ alone is not exactly autonomous due to the dependence of the Hessian matrix on $\mW(t)$, we could augment $\gS(t)$ with $\mW(t)$ to obtain an autonomous system on the augmented state. We come back to this distinction in \cref{app:proofs-stability}. The unicity argument still holds here (the augmented solution would be unique).} differential equation given a particular choice of initial conditions \citep[][Prop. 7.4.2]{wiggins1990dynamicalsystems}. In other words, if we can find another function $\widetilde{\gS}(t)$ that also satisfies the above differential equation with the initial conditions $\widetilde{\gS}(0) = \mI$, then it means that for all $t$ we have $\widetilde{\gS}(t) = \gS(t)$. Suppose that this function can be written as
    \begin{equation*}
        \widetilde{\gS}(t) = \mI - \sum_{i=1}^{M}\sum_{j=1}^{M}\mB_{t}[i,j] \otimes \vphi_{i}\vphi_{j}^{\top},
    \end{equation*}
    where $\mB_{t}[i, j]$ satisfies the dynamical system defined in the statement of \cref{prop:decomposition-W0}. Then we have, using \cref{lem:decomposition-gradient-hessian}:
    \begingroup
    \allowdisplaybreaks
    \begin{align*}
        &-\nabla^{2}\gL\big(\mW(t);f_{\Phi}(\Dtrain)\big)\widetilde{\gS}(t) + \lambda \mI\\
        &\qquad = -\bigg[\sum_{i=1}^{M}\mA_{i}(t)\otimes \vphi_{i}\vphi_{i}^{\top} + \lambda \mI\bigg]\bigg[\mI - \sum_{m,j}\mB_{t}[m,j]\otimes \vphi_{m}\vphi_{j}^{\top}\bigg] + \lambda \mI\\
        &\qquad = -\sum_{i=1}^{M}\mA_{i}(t)\otimes \vphi_{i}\vphi_{i}^{\top} + \sum_{i,j,m}\big(\mA_{i}(t)\mB_{t}[m,j]\big)\otimes \big(\vphi_{i}\underbrace{\vphi_{i}^{\top}\vphi_{m}}_{\in\,\sR}\vphi_{j}^{\top}\big) + \lambda \sum_{i, j}\mB_{t}[i, j] \otimes \vphi_{i}\vphi_{j}^{\top}\\
        &\qquad = -\sum_{i,j}\bigg[\mathbbm{1}(i=j)\mA_{i}(t) - \lambda \mB_{t}[i, j] - \mA_{i}(t)\sum_{m=1}^{M}\big(\vphi_{i}^{\top}\vphi_{m}\big)\mB_{t}[m,j]\bigg]\otimes \vphi_{i}\vphi_{j}^{\top}\\
        &\qquad = -\sum_{i,j}\frac{d\mB_{t}[i,j]}{dt}\otimes \vphi_{i}\vphi_{j}^{\top} = \frac{d\widetilde{\gS}}{dt}
    \end{align*}
    \endgroup
    We have shown that $\widetilde{\gS}(t)$ follows the same dynamics as $\gS(t)$. Moreover, using the initial conditions $\mB_{0}[i, j] = \vzero$, it is clear that $\widetilde{\gS}(0) = \mI$, which are the same initial conditions as the equation satisfied by $\gS(t)$. Therefore, we have $\widetilde{\gS}(t) = \gS(t)$ for all $t$, showing the expected decomposition of $\gS(t) = d\mW(t)/d\mW_{0}$.
\end{proof}

\subsection{Decomposition of the Jacobian matrix $d\protectbm{W}(t)/d\protectbm{\phi}_{m}$}
\label{app:proof-phi}
Similar to \cref{prop:decomposition-W0}, there exists a decomposition of the Jacobian matrix $d\mW(T)/d\vphi_{m}$. Recall that this Jacobian matrix appears in the computation of the gradient of the outer-loss wrt. the embedding vectors $\vphi_{m}$, which is necessary in order to compute the meta-gradients wrt. the meta-parameters of the embedding network $f_{\Phi}$ using backpropagation from the last layer of $f_{\Phi}$.
\begin{proposition}
    Let $f_{\Phi}(\Dtrain) = \{(\vphi_{m}, \vy_{m})\}_{m=1}^{M}$ be the embedded training set through the embedding network $f_{\Phi}$. For all $m$, the Jacobian matrix $d\mW(t)/d\vphi_{m}$ can be written as
    \begin{equation*}
        \frac{d\mW(t)}{d\vphi_{m}} = -\bigg[\vs_{m}(t)\otimes \mI_{d} + \sum_{i=1}^{M}\mB_{t}[i,m]\mW_{0}\otimes \vphi_{i} + \sum_{i=1}^{M}\sum_{j=1}^{M}\vz_{t}[i, j, m]\vphi_{j}^{\top}\otimes \vphi_{i}\bigg],
    \end{equation*}
    where $\vs_{m}(t)$ and $\mB_{t}[i, j]$ are the solutions of the ODEs defined in \cref{prop:decomposition-W,prop:decomposition-W0}, and for all $i,j,m$, $\vz_{t}[i,j,m]$ is a vector of length $N$ solution of the following dynamical system:
    \begin{equation*}
        \frac{d\vz_{t}[i,j,m]}{dt} = -\mA_{i}(t)\bigg[\mathbbm{1}(i=j)\vs_{m}(t) + \mathbbm{1}(i=m)\vs_{j}(t) + \lambda \vz_{t}[i, j, m] + \sum_{k=1}^{M}\big(\vphi_{i}^{\top}\vphi_{k}\big)\vz_{t}[k,j,m]\bigg],
    \end{equation*}
    with the initial conditions $\vz_{0}[i, j, m] = \vzero$.
    \label{prop:decomposition-phi}
\end{proposition}

\begin{proof}
    The outline of this proof follows the proof of \cref{prop:decomposition-W0}: we first use the forward sensitivity equations to get the dynamical system followed by the Jacobian matrix $d\mW(t)/d\vphi_{m}$, and then use a unicity argument given that the new decomposition satisfies the same ODE. Recall that we use the following notation to write the gradient vector field followed during adaptation:
    \begin{equation*}
        \frac{d\mW}{dt} = g\big(\mW(t);f_{\Phi}(\Dtrain)\big)\triangleq -\nabla \gL\big(\mW(t);f_{\Phi}(\Dtrain)\big).
    \end{equation*}
    For a fixed $m$, we introduce the matrix valued function $\gS(t) = d\mW(t)/d\vphi_{m}$ as a sensitivity state. We can use the forward sensitivity equations, and see that $\gS(t)$ satisfies the following equation
    \begin{align*}
        \frac{d\gS}{dt} &= \frac{\partial g\big(\mW(t), \vphi_{m}\big)}{\partial \mW(t)}\gS(t) + \frac{\partial g\big(\mW(t), \vphi_{m}\big)}{\partial \vphi_{m}} & \gS(0) &= \frac{d\mW(0)}{d\vphi_{m}} = \vzero\\
        &= -\nabla^{2}\gL\big(\mW(t);f_{\Phi}(\Dtrain)\big)\gS(t) - \mA_{m}(t)\mW_{0}\otimes \vphi_{m} \\
        & \qquad + \mA_{m}(t)\sum_{i=1}^{M}\vs_{i}(t)\vphi_{i}^{\top}\otimes \vphi_{m} - \frac{1}{M}\big(\vp_{m}(t) - \vy_{m}\big)\otimes \mI_{d}
    \end{align*}
    where we make the direct dependence of $g$ on $\vphi_{m}$ explicit, and we used the decomposition of $\mW(t)$ from \cref{prop:decomposition-W}. Suppose that we define the function $\widetilde{\gS}(t)$ as
    \begin{equation*}
        \widetilde{\gS}(t) = -\bigg[\vs_{m}(t) \otimes \mI_{d} + \sum_{i=1}^{M}\mB_{t}[i,m]\mW_{0}\otimes\vphi_{i} + \sum_{i,j}\vz_{t}[i,j,m]\vphi_{j}^{\top}\otimes \vphi_{i}\bigg],
    \end{equation*}
    where $\vz_{t}[i,j,m]$ satisfies the dynamical system defined in the statement of \cref{prop:decomposition-phi}. Then we have, using \cref{lem:decomposition-gradient-hessian}:
    \begin{align*}
        & -\!\nabla^{2}\gL\big(\mW(t);f_{\Phi}(\Dtrain)\big)\widetilde{\gS}(t) - \mA_{m}(t)\bigg[\mW_{0} -  \sum_{i=1}^{M}\vs_{i}(t)\vphi_{i}^{\top}\bigg]\otimes \vphi_{m} - \frac{1}{M}\big(\vp_{m}(t) - \vy_{m}\big)\otimes \mI_{d}\\
        =\ & \bigg[\sum_{i=1}^{M}\mA_{i}(t)\otimes \vphi_{i}\vphi_{i}^{\top} + \lambda \mI\bigg]\bigg[\vs_{m}(t)\otimes \mI_{d} + \sum_{i=1}^{M}\mB_{t}[i,m]\mW_{0}\otimes \vphi_{i} + \sum_{i,j}\vz_{t}[i,j,m]\vphi_{j}^{\top}\otimes \vphi_{i}\bigg]\\
        & \qquad -\mA_{m}(t)\mW_{0}\otimes \vphi_{m} + \mA_{m}(t)\sum_{i=1}^{M}\vs_{i}(t)\vphi_{i}^{\top}\otimes \vphi_{m} - \frac{1}{M}\big(\vp_{m}(t) - \vy_{m}\big)\otimes \mI_{d}\\
        =\ & -\bigg[\underbrace{\vphantom{\sum_{i=1}^{M}}\frac{1}{M}\big(\vp_{m}(t) - \vy_{m}\big) - \lambda \vs_{m}(t)}_{=\,d\vs_{m}/dt}\bigg]\otimes \mI_{d}\\
        & \qquad - \sum_{i=1}^{M}\bigg[\underbrace{\mathbbm{1}(i = m)\mA_{i}(t) - \lambda \mB_{t}[i, m] - \mA_{i}(t)\sum_{k=1}^{M}\big(\vphi_{i}^{\top}\vphi_{k}\big)\mB_{t}[k,m]}_{=\,d\mB_{t}[i,m]/dt}\bigg]\mW_{0}\otimes \vphi_{i}\\
        & \qquad + \sum_{i,j}\underbrace{\mA_{i}(t)\bigg[\mathbbm{1}(i=j)\vs_{m}(t) + \mathbbm{1}(i=m)\vs_{j}(t) + \lambda \vz_{t}[i, j, m] + \sum_{k=1}^{M}\big(\vphi_{i}^{\top}\vphi_{k}\big)\vz_{t}[k,j,m]\bigg]}_{=\,-d\vz_{t}[i,j,m]/dt}\vphi_{j}^{\top}\otimes \vphi_{i}\\
        =\ & \frac{d\widetilde{\gS}}{dt}
    \end{align*}
    We have shown that $\widetilde{\gS}(t)$ follows the same dynamics as $\gS(t)$. Moreover using the initial conditions $\vs_{m}(0) = \vzero$, $\mB_{0}[i, j] = \vzero$, and $\vz_{0}[i, j, m] = \vzero$, we have $\widetilde{\gS}(0) = \vzero$, which are the same initial conditions as the equation satisfied by $\gS(t)$. Therefore we have $\widetilde{\gS}(t) = \gS(t)$ for all $t$, showing the expected decomposition of the Jacobian matrix $\gS(t) = d\mW(t) / d\vphi_{m}$.
\end{proof}

\subsection{Proof of the meta-gradient wrt. $T$}
\label{app:proof-T}
The novelty of COMLN over prior work on gradient-based meta-learning is the ability to meta-learn the amount of adaptation $T$ using SGD. To compute the meta-gradient wrt. the time horizon $T$, we can apply the result from \citet{chen2018neuralode}.

\begin{proposition}
    Let $f_{\Phi}(\Dtrain)$ and $f_{\Phi}(\Dtest)$ be the embedding through the network $f_{\Phi}$ of the training and test set respectively. The gradient of the outer-loss wrt. the time horizon $T$ is given by:
    \begin{equation*}
        \frac{d\gL\big(\mW(T);f_{\Phi}(\Dtest)}{dT} = -\bigg[\frac{\partial \gL\big(\mW(T);f_{\Phi}(\Dtest)\big)}{\partial \mW(T)}\bigg]^{\top}\frac{\partial \gL\big(\mW(T);f_{\Phi}(\Dtrain)\big)}{\partial \mW(T)}.
    \end{equation*}
    \label{prop:proof-meta-gradient-T}
\end{proposition}

\begin{proof}
    We can directly apply the result from \citep[][App. B.2]{chen2018neuralode}, which we recall here for completeness. Given the ODE $d\vz/dt = g(\vz(t))$, the gradient of the loss $\widetilde{\gL}\big(\vz(T)\big)$ wrt. $T$ is given by
    \begin{align*}
        \frac{d\widetilde{\gL}}{dT} &= \va(T)^{\top}g(\vz(T)) & \textrm{where} \quad \va(T) &= \frac{\partial \widetilde{\gL}}{\partial \vz(T)}
    \end{align*}
    is the \emph{initial adjoint state} (see \cref{sec:local-sensitivity-analysis-odes}). In our case we have
    \begin{align*}
        g(\vz(T)) &\equiv -\nabla\gL\big(\mW(T);f_{\Phi}(\Dtrain)\big) & \va(T) &\equiv \frac{\partial\gL\big(\mW(T);f_{\Phi}(\Dtest)\big)}{\partial \mW(T)}.
    \end{align*}
\end{proof}

\section{Projection onto the Jacobian matrices}
\label{app:projection}
Once we have computed the Jacobian $d\mW(T)/d\vtheta$ wrt. some meta-parameter $\vtheta$ (in our case, either the initial conditions $\mW_{0}$ or the embedding vectors $\vphi_{m}$), we only have to project the vector of partial derivatives onto it to obtain the meta-gradients (vector-Jacobian product), using the chain rule:
\begin{equation*}
    \frac{d\gL\big(\mW(T);f_{\Phi}(\Dtest)\big)}{d\vtheta} = \frac{\partial \gL\big(\mW(T);f_{\Phi}(\Dtest)\big)}{\partial \mW(T)}\frac{d\mW(T)}{d\vtheta} + \frac{\partial \gL\big(\mW(T);f_{\Phi}(\Dtest)\big)}{\partial \vtheta}.
\end{equation*}
While we have shown how to decompose the Jacobian matrices in such a way that it only involves quantities that are independent of $d$ the dimension of the embedding vectors $\vphi$, this final projection a priori requires us to form the Jacobian explicitly. Even though this operation is done only once at the end of the integration, this may be overly expensive since the Jacobian matrices scale quadratically with $d$, which can be as high as $d=16,\!000$ in our experiments.

Fortunately, we can perform this projection as a function of $\vs_{m}(T)$, $\mB_{T}[i, j]$, and $\vz_{T}[i, j, m]$, without having to explicitly form the full Jacobian matrix, by exchanging the order of operations.

\begin{proposition}
    Let $f_{\Phi}(\Dtrain) = \{(\vphi_{m}, \vy_{m})\}_{m=1}^{M}$ and $f_{\Phi}(\Dtest)$ be the embedded training set and test set through the embedding network $f_{\Phi}$ respectively. Let $\mB_{T}[i, j]$ be the solution at time $T$ of the differential equation defined in \cref{prop:decomposition-W0}, and $\mV(T) \in \sR^{N\times d}$ the partial derivative of the outer-loss wrt. the adapted parameters $\mW(T)$:
    \begin{align*}
        \mV(T) &= \frac{\partial \gL\big(\mW(T);f_{\Phi}(\Dtest)\big)}{\partial \mW(T)} & \textrm{so that} \quad \frac{d\gL\big(\mW(T);f_{\Phi}(\Dtest)\big)}{d\mW_{0}} &= \mathrm{Vec}\big(\mV(T)\big)^{T}\frac{d\mW(T)}{d\mW_{0}}.
    \end{align*}
    Let $\vphi = [\vphi_{1}, \ldots, \vphi_{M}]^{\top} \in \sR^{M \times d}$ be the design matrix. The gradient of the outer-loss wrt. the initial conditions $\mW_{0}$ can be computed as
    \begin{align*}
        \frac{d\gL\big(\mW(T);f_{\Phi}(\Dtest)\big)}{d\mW_{0}} &= \mV(T) - \mC^{\top}\vphi,
    \end{align*}
    where the rows of $\mC \in \sR^{M \times N}$ are defined by
    \begin{equation*}
        \mC_{j} = \sum_{i=1}^{M}\vphi_{i}^{\top}\mV(T)^{\top}\mB_{T}[i, j].
    \end{equation*}
    \label{prop:projection-W0}
\end{proposition}

\begin{proof}
In this proof, we will make the encoding of the indices for the Jacobian $d\mW(T)/d\mW_{0}$ more explicit, as mentioned in \cref{app:proofs-meta-gradients}. Recall from \cref{prop:decomposition-W0} that this Jacobian matrix can be decomposed as
\begin{equation*}
    \frac{d\mW(T)}{d\mW_{0}} = \mI - \sum_{i, j}\mB_{T}[i, j] \otimes \vphi_{i}\vphi_{j}^{\top}.
\end{equation*}
Introducing the following notations
\begin{align*}
    \mF &\triangleq \sum_{i,j}\mB_{T}[i, j]\otimes \vphi_{i}\vphi_{j}^{\top} \in \sR^{Nd\times Nd} & \mG &\triangleq \mathrm{Vec}\big(\mV(T)\big)^{\top}\mF \in \sR^{Nd},
\end{align*}
for all $l \in \{0, \ldots, N - 1\}$ and $y \in \{0, \ldots, d - 1\}$:
\begingroup
\allowdisplaybreaks
\begin{align*}
    \mG[dl + y] &= \sum_{k=0}^{N-1}\sum_{x=0}^{d-1}\mV_{T}[k, x]\mF[dk + x, dl + y]\\
    &= \sum_{i,j}\sum_{k=0}^{N-1}\bigg[\underbrace{\sum_{x=0}^{d-1}\mV_{T}[k, x]\vphi_{i}[x]}_{=\,[\mV(T)\vphi_{i}]_{k}}\bigg]\mB_{T}[i, j, k, l]\vphi_{j}[y]\\
    &= \sum_{i,j}\bigg[\underbrace{\sum_{k=0}^{N-1}[\mV(T)\vphi_{i}]_{k}\mB_{T}[i,j,k,l]}_{=\,[\vphi_{i}^{\top}\mV(T)^{\top}\mB_{T}[i, j]]_{l}}\bigg]\vphi_{j}[y]\\
    &= \sum_{j=1}^{M}\bigg[\underbrace{\sum_{i=1}^{M}\big[\vphi_{i}^{\top}\mV(T)^{\top}\mB_{T}[i, j]\big]_{l}}_{=\,\mC_{j,l}}\bigg]\vphi_{j}[y] = \big[\mC^{\top}\vphi\big]_{l, y}
\end{align*}
\endgroup
This shows that $\mG$ is equal to $\mC^{\top}\vphi$ up to reshaping. Using the full form of the Jacobian, including $\mI$, concludes the proof.
\end{proof}
An interesting observation is that if the term $\mC^{\top}\vphi$ is ignored in the computation of the meta-gradient, we recover an equivalent of the first-order approximation introduced in \citet{finn17maml}. Similarly, we can show that we can perform the projection onto the Jacobian matrix $d\mW(T)/d\vphi_{m}$ without having to form the explicit $Nd\times d$ matrix.

\begin{proposition}
Let $f_{\Phi}(\Dtrain) = \{(\vphi_{m}, \vy_{m})\}_{m=1}^{M}$ and $f_{\Phi}(\Dtest)$ be the embedded training set and test set through the embedding network $f_{\Phi}$ respectively. Let $\vs_{m}(T)$ be the solution at time $T$ of the differential equation defined in \cref{prop:decomposition-W}, $\mB_{T}[i, j]$ the solution of the one defined in \cref{prop:decomposition-W0}, and $\vz_{T}[i, j, m]$ the solution of the one defined in \cref{prop:decomposition-phi}. Let $\mV(T)\in \sR^{N \times d}$ be the partial derivative of the outer-loss wrt. the adapted parameters $\mW(T)$:
\begin{equation*}
    \mV(T) = \frac{\partial \gL\big(\mW(T);f_{\Phi}(\Dtest)\big)}{\partial \mW(T)}.
\end{equation*}
Let $\vphi = [\vphi_{1}, \ldots, \vphi_{M}]^{\top} \in \sR^{M \times d}$ be the design matrix. The projection of $\mV(T)$ onto the Jacobian matrix $d\mW(T)/d\vphi_{m}$ can be computed as
\begin{equation*}
    \mathrm{Vec}\big(\mV(T)\big)^{\top}\frac{d\mW(T)}{d\vphi_{m}} = -\big[\vs_{m}(T)^{\top}\mV(T) + \mC_{m}\mW_{0} + \mD_{m}\vphi\big],
\end{equation*}
where $\mC_{m} \in \sR^{N}$ and $\mD_{m} \in \sR^{M}$ are defined by
\begin{align*}
    \mC_{m} &= \sum_{i=1}^{M}\vphi_{i}^{\top}\mV(T)^{\top}\mB_{T}[i, m] & \mD_{m,j} &= \sum_{i=1}^{M}\vz_{T}[i,j,m]^{\top}\mV(T)\vphi_{i}
\end{align*}
\label{prop:projection-phi}
\end{proposition}
Note that the definition of $\mC$ in \cref{prop:projection-phi} matches exactly the definition of $\mC$ in \cref{prop:projection-W0}, and therefore we only need to compute this matrix once to perform the projections for both meta-gradients.

\begin{proof}
\begingroup
\allowdisplaybreaks
Recall from \cref{prop:decomposition-phi} that the Jacobian $d\mW(T)/d\vphi_{m}$ can be decomposed as
\begin{equation*}
    \frac{d\mW(T)}{d\vphi_{m}} = -\bigg[\vs_{m}(T) \otimes \mI_{d} + \sum_{i=1}^{M}\mB_{T}[i, m]\mW_{0}\otimes \vphi_{i} + \sum_{i,j}\vz_{T}[i,j,m]\vphi_{j}^{\top}\otimes \vphi_{i}\bigg].
\end{equation*}
We will consider each of these 3 terms separately.
\begin{itemize}
    \item For the first term, let
    \begin{align*}
        \mF_{1} &\triangleq \vs_{m}(T)\otimes \mI_{d} \in \sR^{Nd \times d} & \mG_{1} & \triangleq \mathrm{Vec}\big(\mV(T)\big)^{\top}\mF_{1} \in \sR^{d},
    \end{align*}
    and for $y \in \{0, \ldots, d-1\}$:
    \begin{align*}
        \mG_{1}[y] &= \sum_{k=0}^{N-1}\sum_{x=0}^{d-1}\mV_{T}[k, x]\mF_{1}[dk + x, y] = \sum_{k=0}^{N-1}\sum_{x=0}^{d-1}\mV_{T}[k, x]\mathbbm{1}(x = y)\big[\vs_{m}(T)\big]_{k}\\
        &= \sum_{k=0}^{N-1}\big[\vs_{m}(T)\big]_{k}\mV_{T}[k, y] = \big[\vs_{m}(T)^{\top}\mV(T)\big]_{y}
    \end{align*}
    Therefore, $\mG_{1}$ the projection of $\mV(T)$ onto the first term of the Jacobian matrix is equal to $\vs_{m}(T)^{\top}\mV(T)$.
    
    \item For the second term, let
    \begin{align*}
        \mF_{2} &\triangleq \sum_{i=1}^{M}\mB_{T}[i, m]\mW_{0}\otimes \vphi_{i} \in \sR^{Nd \times d} & \mG_{2} &\triangleq \mathrm{Vec}\big(\mV(T)\big)^{\top}\mF_{2}\in \sR^{d},
    \end{align*}
    and for $y \in \{0, \ldots, d-1\}$:
    \begin{align*}
        \mG_{2}[y] &= \sum_{k=0}^{N-1}\sum_{x=0}^{d-1}\mV_{T}[k, x]\mF_{2}[dk + x, y]\\
        &= \sum_{i=1}^{M}\sum_{l=0}^{N-1}\sum_{k=0}^{N-1}\bigg[\underbrace{\sum_{x=0}^{d-1} \mV_{T}[k, x] \vphi_{i}[x]}_{=\,[\mV(T)\vphi_{i}]_{k}}\bigg]\mB_{T}[i,m,k,l]\mW_{0}[l,y]\\
        &= \sum_{i=1}^{M}\sum_{l=0}^{N-1}\bigg[\underbrace{\sum_{k=0}^{N-1} [\mV(T)\vphi_{i}]_{k}\mB_{T}[i,m,k,l]}_{=\,\big[\vphi_{i}^{\top}\mV(T)^{\top}\mB_{T}[i,m]\big]_{l}}\bigg]\mW_{0}[l, y]\\
        &= \sum_{l=0}^{N-1}\bigg[\underbrace{\sum_{i=1}^{M}\big[\vphi_{i}^{\top}\mV(T)^{\top}\mB_{T}[i, m]\big]_{l}}_{=\,\mC_{m,l}}\bigg]\mW_{0}[l, y] = \big[\mC_{m}\mW_{0}\big]_{y}
    \end{align*}
    $\mG_{2}$ the projection of $\mV(T)$ onto the second term of the Jacobian matrix is equal to $\mC_{m}\mW_{0}$.
    
    \item Finally for the third term, let
    \begin{align*}
        \mF_{3} &\triangleq \sum_{i,j}\vz_{T}[i,j,m]\vphi_{j}^{\top} \otimes \vphi_{i} \in \sR^{Nd \times d} & \mG_{3} &\triangleq \mathrm{Vec}\big(\mV(T)\big)^{\top}\mF_{3} \in \sR^{d}
    \end{align*}
    and for $y \in \{0, \ldots, d - 1\}$:
    \begin{align*}
        \mG_{3}[y] &= \sum_{k=0}^{N-1}\sum_{x=0}^{d-1}\mV_{T}[k, x]\mF_{3}[dk + x, y]\\
        &= \sum_{i,j}\sum_{k=0}^{N-1}\bigg[\underbrace{\sum_{x=0}^{d-1}\mV_{T}[k, x]\vphi_{i}[x]}_{=\,[\mV(T)\vphi_{i}]_{k}}\bigg]\vz_{T}[i, j, m, k]\vphi_{j}[y]\\
        &= \sum_{i,j}\bigg[\underbrace{\sum_{k=0}^{N-1}[\mV(T)\vphi_{i}]_{k}\vz_{T}[i, j, m, k]}_{=\,\vz_{T}[i,j,m]^{\top}\mV(T)\vphi_{i}\in \sR}\bigg]\vphi_{j}[y]\\
        &= \sum_{j=1}^{M}\bigg[\underbrace{\sum_{i=1}^{M}\vz_{T}[i,j,m]^{\top}\mV(T)\vphi_{i}}_{=\,\mD_{m,j}}\bigg]\vphi_{j}[y] = \big[\mD_{m}\vphi\big]_{y}
    \end{align*}
    $\mG_{3}$ the projection of $\mV(T)$ onto the third term of the Jacobian is equal to $\mD_{m}\vphi$.
\end{itemize}
\endgroup
\end{proof}
Note that the projection of $\mV(T)$ as defined in \cref{prop:projection-phi} is not equal to the meta-gradient itself, as it is missing the term coming from the partial derivative of the outer-loss wrt. the embedding vector $\partial \gL\big(\mW(T);f_{\Phi}(\Dtest)\big)/\partial \vphi_{m}$, which is non-zero unlike the counterpart for $\mW_{0}$. However, this projection is the only expensive operation required to compute the total gradient.

\section{Proof of stability}
\label{app:proofs-stability}
In this section, we would like to show that unlike the adjoint method (see \cref{sec:local-sensitivity-analysis-odes,sec:challenges-optimization-meta-learning-objective}), our method to compute the meta-gradients of COMLN is guaranteed to be \emph{stable}. In other words, even under some small perturbations due to the ODE solver, the solution found by numerical integration is going to stay close to the true solution of the dynamical system. To do so, we use the concept of Lyapunov stability of the solution of an ODE, which we recall here for completeness:
\begin{definition}[Lyapunov stability; \citealp{wiggins1990dynamicalsystems}]
    The solution $\bar{\vx}(t)$ of a dynamical system is said to be \emph{Lyapunov stable} if, given $\varepsilon > 0$, there exists $\delta$ such that for any other solution $\vx(t)$ such that $\|\bar{\vx}(0) - \vx(0)\| < \delta$, we have $\|\bar{\vx}(t) - \vx(t)\| < \varepsilon$ for all $t > 0$.
\end{definition}

We would like to understand the conditions needed for a function $f :\sR^n \rightarrow \sR^n$, such that a trajectory $\bar{\vx}(t)$ satisfying the following autonomous differential equation is (Lyapunov) stable:
\begin{equation*}
\frac{d\vx}{dt} = f\big(\vx(t)\big) 
\end{equation*}
If we assume for a moment that $f(\vx) = -\nabla \gL(\vx)$ where $\gL$ is a convex function, then for any two trajectories $\vx_1(t)$ \& $\vx_2(t)$ of the dynamical system above, we can see that the function $F$ defined by 
\begin{equation*}
F(t) = \frac{1}{2} \|\vx_1(t)-\vx_2(t)\|^2 
\end{equation*}
satisfies $\dot{F}(t) \leq 0$, where we use the notation $\dot{F}(t) \triangleq dF/dt$. Indeed, since $\nabla^2 \gL$ is positive semi-definite and by defining $\vh(t) = \vx_2(t)-\vx_1(t)$, we get 
\begin{align}
\dot{F}(t) &= \frac{dF}{dt} = \bigg[\frac{dF}{d\vh}\bigg]^{\top} \frac{d\vh}{dt} \nonumber\\
&= \big(\vx_2(t)-\vx_1(t)\big)^{\top} \big(\dot{\vx}_2(t)-\dot{\vx}_1(t)\big) \nonumber\\
&= \big(\vx_2(t)-\vx_1(t)\big)^{\top}\big(f(\vx_2(t))-f(\vx_1(t))\big) \nonumber\\
&= \vh(t)^{\top} \left[\int_{0}^{1} Df\big(\vx_1(t)+s\vh(t)\big) ds \right] \vh(t) \label{eq:fundamental-theorem-calculus}\\
&= -\int_{0}^{1} \underbrace{\vh(t)^{\top}\nabla^2\gL \big(\vx_1(t)+s\vh(t)\big)\vh(t)}_{\geq 0} ds \nonumber
\end{align}
where \cref{eq:fundamental-theorem-calculus} follows from the \textit{fundamental theorem of calculus}. Now since $\dot{F}(t) \leq 0$, the distance between two trajectories decreases as time $t$ increases, hence we have Lyapunov stability for all trajectories or solutions of the above autonomous differential equation.

Now note that we didn't actually need $f$ to be of the specific form $f(\vx) = -\nabla \gL(\vx)$, but having the Jacobian matrix $Df$ negative semi-definite everywhere would have been sufficient. Hence we get the following statement
\begin{proposition}
If a continuously differentiable function $f :\mathbb{R}^n \rightarrow \mathbb{R}^n$ is such that the Jacobian matrix $Df$ is negative semi-definite everywhere, then any solution $\bar{\vx}(t)$ of the autonomous differential equation
\begin{equation*}
\frac{d\vx}{dt} = f\big(\vx(t)\big) 
\end{equation*}
with initial condition $\bar{\vx}(0) = \vx_0$ is (Lyapunov) stable.
\label{prop:stability-aut-ds}
\end{proposition}

We can extend the above result to the non-autonomous case. For this purpose, let us define 
\begin{equation}
\frac{d\vx}{dt} = f\big(t,\vx(t)\big) 
\label{eq:nonautonomous-ds}
\end{equation}
which induces the parametrized function $f_t : \vx \mapsto f(t,\vx)$. With this notation we can state the following result. 
\begin{proposition}
If the Jacobian matrix $Df_t$ is negative semi-definite everywhere for all $t\geq 0$, then any trajectory $\vx(t)$, solution of \cref{eq:nonautonomous-ds} with initial condition $\vx(0)=\vx_0$, is (Lyapunov) stable.
\label{prop:stability-nonaut-ds}
\end{proposition}
\begin{proof}
Let us start with two trajectories $\vx_1(t)$ and $\vx_2(t)$, which are solutions to the above \cref{eq:nonautonomous-ds}. By defining $\vh(t) = \vx_2(t) - \vx_1(t)$ and by applying the \textit{fundamental theorem of calculus}, we get 
\begin{align*}
    f\big(t, \vx_1(t)\big)-f\big(t,\vx_2(t)\big) &= f_t\big(\vx_1(t)\big)-f_t\big(\vx_2(t)\big)\\
    &= \left[\int_{0}^{1} Df_t\big(\vx_1(t)+s\vh(t)\big) ds \right] \vh(t)
\end{align*}
Following the same idea as already previously outlined, let us define $F(t) = \frac{1}{2}\|\vx_1(t)- \vx_2(t)\|^2$ and show that $\dot{F}(t) \leq 0$: 
\begin{align*}
    \dot{F}(t) &= (\vx_2(t)-\vx_1(t))^{\top} \big(f(t,\vx_2(t))-f(t, \vx_1(t))\big)\\
    &= \vh(t)^{\top}\left[\int_{0}^{1} Df_t\big(\vx_1(t)+s\vh(t)\big) ds \right]\vh(t)\\
    &= \int_{0}^{1} \underbrace{\vh(t)^T Df_t\big(\vx_1(t)+s\vh(t)\big)\vh(t)}_{\leq 0} ds \leq 0
\end{align*}
Hence the distance between two trajectories decreases as time $t$ increases, and we have Lyapunov stability for all solutions of the non-autonomous differential equation above.
\end{proof}

\subsection{Stability of the forward sensitivity equations}
\label{app:stability-general}
In this subsection let us start by considering the general case of solving the initial value problem to the following autonomous system of differential equations:
\begin{equation}
\centering
\left\{\begin{split}
\frac{d\mW}{dt} &= g(\mW(t), \vtheta) & \mW(0)&=\mW_0\\
\frac{d\gS}{dt} &= \frac{\partial g\big(\mW(t), \vtheta\big)}{\partial \mW(t)}\cdot \gS(t) + \frac{\partial g\big(\mW(t), \vtheta\big)}{\partial \vtheta} & \gS(0) &= \frac{\partial \mW(0)}{\partial \vtheta}. \\ 
\end{split}\right.
\label{eq:system-general-case}
\end{equation}
\begin{proposition}
There exists a solution of the autonomous system of differential equations in \cref{eq:system-general-case}, which is also unique.  
\end{proposition}
\begin{proof}
By Theorem 7.1.1 in \cite{wiggins1990dynamicalsystems}, we know that there exists a solution of the system of differential equations in Eq.~\ref{eq:system-general-case}. Then, by Theorem 7.4.1 in \cite{wiggins1990dynamicalsystems}, we can conclude that this solution is unique upon the choice of initial conditions $\mW(0)=\mW_0$ and $\gS(0) = \partial \mW(0) / \partial \vtheta$.\\

Alternatively, one could also apply Theorems 7.1.1 \& 7.4.1 in \cite{wiggins1990dynamicalsystems} to conclude that the initial value problem $d\mW/dt = g(\mW(t), \vtheta)$ with initial condition $\mW(0)=\mW_0$ has a unique solution. This existence and uniqueness then gives rise to existence and uniqueness of a solution of the entire system of differential equations via $\gS(t) = d\mW/dt$ by applying Clairnaut's theorem. 
\end{proof}

In our case, we have $g(\mW(t), \theta) = -\nabla \gL\big(\mW(T), \vtheta \big)$, where $\gL$ is a convex function in $\mW$. Hence the autonomous system of differential equations can be rewritten as 
\begin{equation}
\centering
\left\{\begin{split}
\frac{d\mW}{dt} &= -\nabla \gL\big(\mW(T), \vtheta \big) & \mW(0)&=\mW_0\\
\frac{d\gS}{dt} &= -\nabla^2 \gL\big(\mW(T), \vtheta \big) \gS(t) + \frac{\partial g\big(\mW(t), \vtheta\big)}{\partial \vtheta} & \gS(0) &= \frac{\partial \mW(0)}{\partial \vtheta}. \\ 
\end{split}\right.
\label{eq:system-general-Lspecific}
\end{equation}
\begin{proposition}
Let $\mW(t)$ be the solution of
\begin{align*}
    \frac{d\mW}{dt} &= -\nabla \gL\big(\mW(T), \vtheta \big) & \qquad \mW(0)&=\mW_0
\end{align*}
Then the trajectory $\mW(t)$ is (Lyapunov) stable. 
\end{proposition}
\begin{proof}
Note that $Dg\big(\mW(t)\big) = -\nabla^2 \gL$ here. Since $\gL$ is convex, $\nabla^2 \gL$ is positive semi-definite, and hence $Dg(\mW(t))$ is negative semi-definite. Hence the result directly follows from \cref{prop:stability-aut-ds}.
\end{proof}
In addition we can now also conclude that the solution of \cref{eq:system-general-Lspecific} is also (Lyapunov) stable.
\begin{cor}
The solution $\big[\mW(t), \gS(t)\big]$ of \cref{eq:system-general-Lspecific} is (Lyapunov) stable.
\end{cor}
This result is general to the application of the forward sensitivity equations on a gradient vector field derived from a convex loss function. We can also show that when the gradient $d\mW^{\star}/d\vtheta$ exists and is finite (recall that $\mW^{\star}$ is the minimizer of the loss function, and the equilibrium of the gradient vector field), then the solution of the system is also bounded, which guarantees us that our solution will not diverge (unlike the adjoint method applied to a gradient vector field).

\begin{proposition}
Assuming that $\big\|d\mW^{\star} / d \vtheta\big\|$ is finite, the solution $\big[\mW(t),\gS(t)\big]$ of Eq. \ref{eq:system-general-Lspecific} is bounded. 
\end{proposition}
\begin{proof} Let us define
\begin{align*}
        V(t) &= \frac{1}{2}\|\mW(t)\|^2 + \frac{1}{2}\|\gS(t)\|^2
    \end{align*}
Since $\mW(t)$ and $\gS(t)$ are continuous functions in $t$, verifying 
\begin{align*}
        \mW(t)&\underset{t \rightarrow \infty}{\longrightarrow} \mW^{\star} & \textrm{and} & \quad & \gS(t) \underset{t \rightarrow \infty}{\longrightarrow} \frac{d\mW^{\star}}{d \vtheta}
    \end{align*}
Hence
\begin{equation*}
    V(t) \underset{t \rightarrow \infty}{\longrightarrow} \frac{1}{2}\|\mW^{\star}\|^2 + \frac{1}{2}\Big\|\frac{d\mW^{\star}}{d \vtheta}\Big\|^2 < \infty,
\end{equation*}
where the last point follows from our assumption that $\big\|d\mW(t) / d \vtheta\big\| < \infty$. We also have that  
\begin{equation*}
    V(0) = \frac{1}{2}\|\mW_0\|^2 + \frac{1}{2}\Big\|\frac{d\mW(0)}{d \vtheta}\Big\|^2 < \infty.
\end{equation*}
$V(t)$ being continuous as a composition of continuous functions, we conclude that $V$ is bounded, and hence $\mW(t)$ and $\gS(t)$ are bounded as well. 
\end{proof}

\section{Experimental details}
\label{app:experiment-details}

\subsection{Choice of the numerical solver}
\label{app:choice-numerical-solver}
In \cref{sec:empirical-efficiency}, we showed that the numerical solver had a significant impact on the time to compute the meta-gradients (but not the memory requirements). In \cref{tab:numerical-solver}, we show that using explicit Euler instead of an adaptive scheme like Runge-Kutta leads to very similar results. In all our experiments we therefore chose Runge-Kutta for its faster execution, and by convenience---it is the default numerical solver available in JAX \citep{jax2018github}.

\begin{table}[t]
    \caption{The effect of the numerical solver on the performance of COMLN. The average accuracy (\%) on $1,000$ held-out meta-test tasks is reported with $95\%$ confidence interval. Note that for a given setting, the same $1,000$ tasks are used for evaluation, making both methods directly comparable. RK: 4th-order Runge-Kutta with Dormand Prince adaptive step size. Euler: explicit Euler scheme.}
    \label{tab:numerical-solver}
    \centering
    \begin{tabular}{lcc}
        \toprule
        \multirow{2}{*}{\textbf{Method}} & \multicolumn{2}{c}{\textbf{\textit{mini}ImageNet 5-way}}\\
        \cmidrule(lr){2-3}
         & \textbf{1-shot} & \textbf{5-shot}\\
        \midrule
        COMLN (RK) & $53.01 \pm 0.62$ & $70.54 \pm 0.54$\\
        COMLN (Euler) & $53.00 \pm 0.83$ & $70.50 \pm 0.72$\\
        \bottomrule
    \end{tabular}
\end{table}

\subsection{Details for measuring the efficiency}
\label{app:details-measuring-efficiency}
In Figures \ref{fig:conv4_memory_runtime} \& \ref{fig:resnet12_memory_runtime}, we show the efficiency, both in terms of memory and runtime, of COMLN compared to other meta-learning algorithms. The computational efficiency (runtime) was measured as the average time taken to compute the meta-gradients on a single 5-shot 5-way task from the \textit{mini}ImageNet dataset over 100 runs.

\begin{figure}[h]
    \centering
    \begin{tikzpicture}[]

\pgfkeys{
    /pgf/number format/.cd,
    sci,
    sci generic={mantissa sep=\times,exponent={10^{#1}}}
}

\begin{groupplot}[
    group style={group size=2 by 2, horizontal sep=5em}
]

\nextgroupplot[
    ymajorgrids,
    ymode=log,
    xmode=log,
    ylabel={Memory usage (in Gb)},
    xlabel={Number of gradient steps},
    ymax=32,
    xminorticks=false,
    ymin=3,
    ytick={1e1},
    minor ytick={3e0, 4e0, 5e0, 6e0, 7e0, 8e0, 9e0, 2e1, 3e1},
    height=6cm,
    width=9cm,
    legend style={font=\small},
]

\addplot[color=tableau10_C0, mark=*, mark size=1.5pt, thick] coordinates {
    (1, 7.496)
};

\addplot[color=tableau10_C1, mark=x, thick] coordinates {
    (1, 7.496)
    (5, 7.496)
    (10, 7.496)
    (25, 7.496)
    (50, 7.496)
    (100, 7.496)
    (500, 7.496)
    (1000, 7.496)
    (5000, 7.496)
    (10000, 7.496)
    (100000, 7.496)
    (1000000, 7.496)
    (10000000, 7.496)
    (100000000, 7.496)
};

\addplot[color=tableau10_C2, mark=+, thick] coordinates {
    (1, 4.637)
    (5, 4.638)
    (10, 4.640)
    (25, 4.644)
    (50, 4.652)
    (100, 4.668)
    (500, 4.790)
    (1000, 4.944)
    (5000, 6.173)
    (10000, 7.709)
};

\addplot[color=tableau10_C3, mark=triangle*, thick] coordinates {
    (1, 4.638)
    (5, 4.638)
    (10, 4.638)
    (25, 4.638)
    (50, 4.638)
    (100, 4.638)
    (500, 4.638)
    (1000, 4.638)
    (5000, 4.638)
    (10000, 4.638)
    (100000, 4.638)
    (1000000, 4.638)
    (10000000, 4.638)
    (100000000, 4.638)
};

\addplot[color=tableau10_C0, dashed, thick] coordinates {(1, 7.496) (5, 32)};
\addplot[color=tableau10_C2, dashed, thick] coordinates {(10000, 7.709) (100000, 32)};

\legend{MAML, iMAML, ANIL, COMLN}

\nextgroupplot[
    ymajorgrids,
    ymode=log,
    xmode=log,
    ylabel={Runtime (in ms)},
    xlabel={Number of gradient steps},
    xminorticks=false,
    log identify minor tick positions=true,
    height=6cm,
    width=9cm,
    legend style={font=\small},
    ymin=100,
    ymax=2.5e3,
    enlarge y limits=true
]

\addplot[color=tableau10_C0, mark=*, mark size=1.5pt, thick] coordinates {
    (1, 376.40)
};

\addplot[color=tableau10_C2, mark=+, thick] coordinates {
    (1, 177.69)
    (5, 179.30)
    (10, 181.13)
    (100, 202.79)
    (1000, 410.08)
    (10000, 2169.17)
};

\addplot[color=tableau10_C3, mark=x, thick, dashed, mark options={solid}] coordinates {
    (10, 170.24)
    (100, 192.45)
    (1000, 342.63)
    (10000, 1799.53)
};
\addplot[color=tableau10_C3, mark=triangle*, thick] coordinates {
    (10, 172.48)
    (100, 182.98)
    (1000, 185.28)
    (10000, 194.74)
    (100000, 210.66)
    (1000000, 223.66)
    (10000000, 237.41)
    (100000000, 255.32)
};

\addplot[color=tableau10_C0, dashed, thick] coordinates {(1, 376.40) (5, 3.5e3)};
\addplot[color=tableau10_C2, dashed, thick] coordinates {(10000, 2169.17) (16000, 3.5e3)};

\legend{MAML, ANIL, COMLN (Euler), COMLN (RK)}

\nextgroupplot[
  at=(group c1r1.north),
  axis y line=none,
  axis x line*=top,
  x axis line style=-,
  grid=none,
  xmode=log,
  xlabel={$T$},
  xmin=0.01,
  xmax=1000000,
  xminorticks=false,
  xtick={0.01, 0.1, 1, 10, 100, 1000, 10000, 100000, 1000000},
  height=2cm,
  width=9cm,
  enlarge x limits=true,
  every axis x label/.style={
    at={(ticklabel* cs:0.)},
    anchor=south east,
  },
  axis line style={draw=none},
  x tick style={draw=none},
  ymin=0,
  ymax=1
]

\nextgroupplot[
  at=(group c2r1.north),
  axis y line=none,
  axis x line*=top,
  x axis line style=-,
  grid=none,
  xmode=log,
  xlabel={$T$},
  xmin=0.01,
  xmax=1000000,
  xminorticks=false,
  xtick={0.01, 0.1, 1, 10, 100, 1000, 10000, 100000, 1000000},
  height=2cm,
  width=9cm,
  enlarge x limits=true,
  every axis x label/.style={
    at={(ticklabel* cs:0.)},
    anchor=south east,
  },
  axis line style={draw=none},
  x tick style={draw=none},
  ymin=0,
  ymax=1
]

\end{groupplot}

\end{tikzpicture}
    \caption{Empirical efficiency of COMLN on a single 5-shot 5-way task, with a ResNet-12 backbone; this figure is similar to \cref{fig:conv4_memory_runtime}. (Left) Memory usage for computing the meta-gradients as a function of the number of inner-gradient steps. The extrapolated dashed lines correspond to the method reaching the memory capacity of a Tesla V100 GPU with 32Gb of memory. (Right) Average time taken (in ms) to compute the exact meta-gradients. The extrapolated dashed lines correspond to the method taking over 3 seconds.}
    \label{fig:resnet12_memory_runtime}
\end{figure}
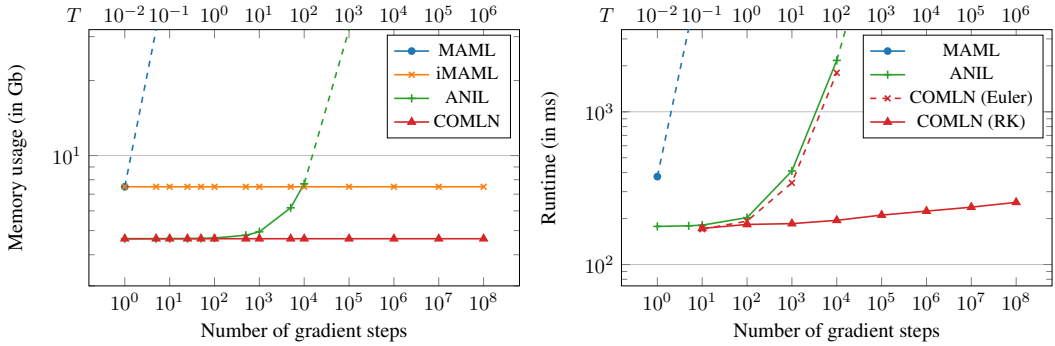

In order to ensure fair comparison between methods that rely on a discrete number of gradient steps (MAML, ANIL, and iMAML), and COMLN which relies on a continuous integration time $T$, we added a conversion between the number of gradient steps and $T$. This corresponds to taking a learning rate of $\alpha = 0.01$ in \cref{eq:gradient-based-meta-learning-adaptation} (which is standard practice for MAML and ANIL on \text{mini}ImageNet). This means that the number of gradient steps is $100\times$ larger than $T$. This can be formally justified by considering an explicit Euler scheme for COMLN (see \cref{sec:connection-anil}), with a constant step size $\alpha = 0.01$. The memory requirements is independent of the choice of the numerical solver. For the computation time, this correspondence between $T$ and the number of gradient steps is no longer exact for \emph{COMLN (RK)}---the comparison with \emph{COMLN (Euler)} is still valid though.

\subsection{Analysis of the learned horizon}
\label{app:analysis-learned-horizon}
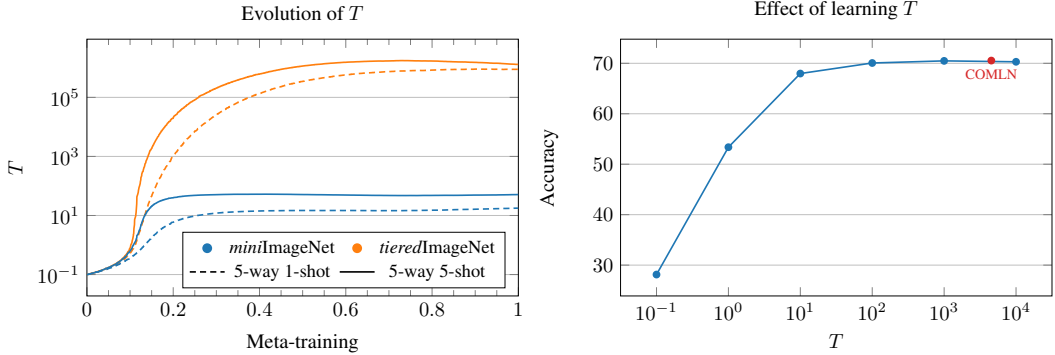
\begin{figure}[t]
    \centering
    \begin{tikzpicture}[]

\pgfplotsset{
    legend image with text/.style={
        legend image code/.code={%
            \node[anchor=center] at (0.3cm,0cm) {#1};
        }
    },
}

\begin{groupplot}[
    group style={group size=2 by 2, horizontal sep=5em}
]

\nextgroupplot[
    ymajorgrids,
    title={Evolution of $T$},
    ymode=log,
    ylabel={$T$},
    xlabel={Meta-training},
    xminorticks=false,
    minor x tick num=3,
    minor y tick num=20,
    height=6cm,
    width=9cm,
    legend columns=2,
    legend style={font=\small},
    legend pos=south east,
    xmin=0., xmax=1.
]
\addlegendimage{only marks, mark=*, tableau10_C0, thick}
\addlegendimage{only marks, mark=*, tableau10_C1, thick}
\addlegendimage{black, thick, densely dashed}
\addlegendimage{black, thick}

\addplot[color=tableau10_C1, mark=none, thick, densely dashed] table[col sep=comma, x=step, y=tieredimagenet_1] {figures/pgfplots/data/integration_time.csv};

\addplot[color=tableau10_C1, mark=none, thick] table[col sep=comma, x=step, y=tieredimagenet_5] {figures/pgfplots/data/integration_time.csv};

\addplot[color=tableau10_C0, mark=none, thick, densely dashed] table[col sep=comma, x=step, y=miniimagenet_1] {figures/pgfplots/data/integration_time.csv};

\addplot[color=tableau10_C0, mark=none, thick] table[col sep=comma, x=step, y=miniimagenet_5] {figures/pgfplots/data/integration_time.csv};

\legend{\textit{mini}ImageNet, \textit{tiered}ImageNet, 5-way 1-shot, 5-way 5-shot}

\nextgroupplot[
    ymajorgrids,
    title={Effect of learning $T$},
    xmode=log,
    ylabel={Accuracy},
    xlabel={$T$},
    ytick={30,40,50,60,70},
    xminorticks=false,
    height=6cm,
    width=9cm,
    legend style={font=\small},
    enlarge x limits=0.1,
    xmin=0.1, xmax=10000.0
]

\addplot[color=tableau10_C0, mark=*, mark size=1.5pt, thick] table [] {
    x       y
    0.1     28.13
    1.0     53.37
    10.0    67.97
    100.0   70.06
    1000.0  70.48
    10000.0 70.30
};

\addplot[color=tableau10_C3, only marks, mark=*, mark size=1.5pt, thick,
    visualization depends on=\thisrow{alignment} \as \alignment,
    nodes near coords, % Place nodes near each coordinate
    point meta=explicit symbolic, % The meta data used in the nodes is not explicitly provided and not numeric
    every node near coord/.style={anchor=\alignment, font=\scriptsize, color=tableau10_C3, fill=white, inner sep=1pt, yshift=-3pt} % Align each coordinate at the anchor 40 degrees clockwise from the right edge
    ] table [% Provide data as a table
     meta index=2 % the meta data is found in the third column
    ] {
    x       y       label   alignment
    4532.58  70.54   COMLN   90
};

\end{groupplot}

\end{tikzpicture}
    \caption{(Left) Evolution of the meta-parameter $T$ controlling the amount of adaptation necessary for all tasks during meta-training. Here, the backbone is a ResNet-12. We normalized the duration of meta-training in $[0, 1]$ to account for early stopping; typically the model for \textit{mini}ImageNet requires an order of magnitude fewer iterations. (Right) Comparison of COMLN (where $T$ is learned, in red) to meta-learning with a fixed length of adaptation $T$ (in blue), on a 5-shot 5-way classification problem on the \textit{mini}ImageNet dataset (with a Conv-4 backbone).}
    \label{fig:experiments}
\end{figure}
Since one of the advantage of COMLN compared to other gradient-based methods is its capacity to learn the amount of adaptation through the time horizon $T$, \cref{fig:experiments} (left) shows the evolution of this meta-parameter during meta-training. We can observe that for the more complex dataset \textit{tiered}ImageNet, COMLN appropriately learns to use a longer sequence of adaptation. Similarly within each dataset, it also learns to use shorter sequences of adaptation of 1-shot problems, possibly to allow for better generalization and to reduce overfitting.

Besides adapting the amount of adaptation to the problem at hand, learning $T$ also has the advantage of saving computation while reaching high levels of performance. If we were to fix $T$ ahead of meta-training (as is typically the case in gradient-based meta-learning, where the number of gradient steps for adaptation is a hyperparameter) to a large value in order to reach high accuracy, as is shown in \cref{fig:experiments} (right), then it would induce larger computational costs early on during meta-training compared to COMLN, which achieves equal performance while tuning the value of $T$. In COMLN, the value of $T$ is relatively small at the beginning of meta-training.

\subsection{Additional Experiment: Preprocessed \textit{mini}ImageNet dataset}
\label{app:additional-experiments}
In addition to our experiments on the \textit{mini}ImageNet and \textit{tiered}ImageNet datasets in \cref{sec:experiments}, we want to evaluate the performance of continuous-time adaptation in isolation from learning the embedding network. We evaluate this using a preprocessed version of \textit{mini}ImageNet introduced in \citet{rusu2018leo}, where the embeddings were trained using a Wide Residual Network (WRN) via supervised classification on the meta-train set. For our purposes, we consider these embeddings $\vphi \in \sR^{640}$ as fixed, and only meta-learn the initial conditions $\mW_{0}$ as well as $T$. 

The classification accuracies for COMLN and other baselines using pretrained embeddings are shown in \cref{tab:results-leo-embeddings}. 
COMLN achieves comparable or better performance with a single linear classifier layer as other meta-learning methods in both the 5-way 1-shot and 5-way 5-shot classification tasks. The single exception is the 5-way 1-shot result of Meta-SGD from \citet{rusu2018leo}, which exceeded the performance of COMLN. However, our implementation of Meta-SGD achieved comparable performance to COMLN. This gap in performance is likely due to the additional data used in \citet{rusu2018leo} (meta-train and meta-validation splits) during meta-training, as opposed to using only the meta-training set for all other baselines. These results show that isolated from representation learning, all these meta-learning algorithms (either gradient-based or not) perform similarly, and COMLN is no exception.

One notable exception though is MetaOptNet \citep{lee2019metaoptnet}, where the performance is not as high as the other baselines when the backbone network is not learned anymore---despite often being the best performing model in \cref{tab:results-architecture}. Our hypothesis is that this discrepancy is due to the accuracy of the QP solver used in MetaOptNet, since learning individual SVMs on 1,000 held-out meta-test tasks leads to performance matching all other methods (about 50\% for 5-way 1-shot and about 70\% for 5-way 5-shot).

\begin{table}[t]
\centering
\caption{\textit{mini}ImageNet results using LEO embeddings and a single linear classifier layer. The average accuracy (\%) on 1,000 held-out meta-test tasks is reported with $95\%$ confidence interval. * Results reported in \citep{rusu2018leo}. ** Note that LEO uses more than a single linear classifier layer, but we add the numbers for completeness.}
\label{tab:results-leo-embeddings}
\begin{tabular}{lcc}
    \toprule
    \multirow{2}{*}{\textbf{Model}} & \multicolumn{2}{c}{\textbf{\textit{mini}ImageNet 5-way}}\\
     \cmidrule(lr){2-3}
     & \textbf{1-shot} & \textbf{5-shot}\\
    \midrule
    MAML \citep{finn17maml} & $50.35 \pm 0.63$ & $65.28 \pm 0.54$\\
    $\text{Meta-SGD}^*$ \citep{li2017metasgd} & $54.24 \pm 0.03$ & $70.86 \pm 0.04$\\
    Meta-SGD \citep{li2017metasgd} & $50.57 \pm 0.64$ & $69.09 \pm 0.53$\\
    iMAML \citep{rajeswaran2019imaml} & $50.26 \pm 0.61 $ & $69.52 \pm 0.51 $ \\
    R2D2 \citep{bertinetto2018r2d2} & $50.33 \pm 0.62$ & $70.38 \pm 0.52$\\
    LRD2 \citep{bertinetto2018r2d2} & $50.41 \pm 0.62 $ & $70.29 \pm 0.52$\\
    LEAP \citep{flennerhag2018leap} & $50.95 \pm 0.62$ & $66.72 \pm 0.55$\\
    MetaOptNet \citep{lee2019metaoptnet} & $40.60 \pm 0.60$ & $50.94 \pm 0.62$\\
    \midrule
    LEO** \citep{rusu2018leo} & $61.76 \pm 0.08$ & $77.59 \pm 0.12$\\
    \midrule
    \textbf{COMLN (Ours)} & $50.39 \pm 0.63$ & $70.06 \pm 0.52$\\
    \bottomrule
\end{tabular}
\end{table}

\subsection{Experimental details}
\label{app:experimental-details-sub}
For all methods and all datasets, we used SGD with momentum $0.9$ and Nesterov acceleration, with a decreasing learning rate starting at $0.1$ and decreasing according to the schedule provided by \citet{lee2019metaoptnet}. For meta-training, we followed the standard procedure in gradient-based meta-learning, and meta-trained with a fixed number of shots: for example in \textit{mini}ImageNet 5-shot 5-way, we only used tasks with $k = 5$ training examples for each of the $N = 5$ classes. This contrasts with \citet{lee2019metaoptnet}, which uses a larger number of shots during meta-training than the one used for evaluation (e.g. meta-training with $k=15$, and evaluating on $k=1$). This may explain the gap in performance between COMLN and MetaOptNet, especially on 1-shot settings. We opted to not follow this decision made by \citet{lee2019metaoptnet} to ensure a fair comparison with other gradient-based methods, which all used the process described above.

\paragraph{Conv-4 backbone} We used a standard convolutional neural network with 4 convolutional blocks \citep{finn17maml}. Each block consists of a convolutional layer with a $3\times 3$ kernel and $64$ channels, followed by a batch normalization layer, and a max-pooling layer with window size and stride $2 \times 2$. The activation function is a ReLU.

\paragraph{ResNet-12 backbone} We largely followed the architecture from \citet{lee2019metaoptnet}, which consists of a 12-layer residual network. The neural network is composed of 4 blocks with residual connections of 3 convolutional layers with a $3\times 3$ kernel. The convolutional layers in the residual block have $k=64$, $160$, $320$ and $640$ channels respectively. The non-linearity functions are \textsc{LeakyReLU}$(0.1)$, and a max-pooling layer with window size and stride $2 \times 2$ is applied at the end of each block. No global pooling is performed at the end of the embedding network, meaning that the embedding dimension is $d = 16,\!000$. The only notable difference with the architecture used by \citet{lee2019metaoptnet} is the absence of DropBlock \citep{ghiasi2018dropblock} regularization.

\end{document}